\def\year{2019}\relax
\documentclass[letterpaper]{article}
\usepackage{aaai19}
\usepackage{times}
\usepackage{helvet}
\usepackage{courier}
\usepackage{stmaryrd}
\usepackage{url}
\usepackage{graphicx}
\usepackage{amssymb}
\usepackage{macros}
\usepackage{thm-restate}
\usepackage{tabularx}
\usepackage{multirow}
\usepackage{etoolbox}
\usepackage{thm-restate}
\usepackage{algorithm}
\usepackage[noend]{algpseudocode}
\makeatletter
\algrenewcommand\ALG@beginalgorithmic{\small}
\makeatother

\newtoggle{withappendix}
\toggletrue{withappendix}

\begin{document}

\title{Modular Materialisation of Datalog Programs}
\author{
    Pan Hu \and Boris Motik \and Ian Horrocks \\ 
    Department of Computer Science, University of Oxford \\
    Oxford, United Kingdom \\
    firstname.lastname@cs.ox.ac.uk
}

\maketitle

\begin{abstract}
The semina\"ive algorithm can be used to materialise all consequences of a
datalog program, and it also forms the basis for algorithms that incrementally
update a materialisation as the input facts change. Certain (combinations of)
rules, however, can be handled much more efficiently using custom algorithms.
To integrate such algorithms into a general reasoning approach that can handle
arbitrary rules, we propose a modular framework for computing and maintaining a
materialisation. We split a datalog program into modules that can be handled
using specialised algorithms, and we handle the remaining rules using the
semina\"ive algorithm. We also present two algorithms for computing the
transitive and the symmetric--transitive closure of a relation that can be used
within our framework. Finally, we show empirically that our framework can
handle arbitrary datalog programs while outperforming existing approaches,
often by orders of magnitude.

\end{abstract}

\section{Introduction}\label{sec:introduction}

Datalog \cite{abiteboul1995foundations} is a prominent rule language whose
popularity is mainly due to its ability to express recursive definitions such
as transitive closure. Datalog captures OWL 2 RL \cite{motik2009owl} ontologies
with SWRL rules \cite{horrocks2004swrl}, so it supports query answering on the
Semantic Web. It has been implemented in many systems, including but not
limited to WebPIE \cite{urbani2012webpie}, VLog \cite{urbani2016column},
Oracle's RDF Store \cite{wu2008implementing}, OWLIM \cite{bishop2011owlim}, and
RDFox \cite{nenov2015rdfox}.

Datalog reasoning is often realised by precomputing and storing all
consequences of a datalog program and a set of facts; this process and its
output are both called \emph{materialisation}. A materialisation must be
updated when the input facts change, but doing so `from scratch' can be
inefficient if changes are small. To minimise the overall work,
\emph{incremental maintenance algorithms} have been developed. These include
the well-known \emph{Delete/Rederive} (DRed)
\cite{gupta1993maintaining,staudt1996incremental} and \emph{Counting}
\cite{gupta1993maintaining} algorithms, and the more recent
\emph{Backward/Forward} (B/F) \cite{motik2015incremental}, DRed$^c$, and
B/F$^c$ \cite{DBLP:conf/aaai/HuMH18} algorithms.

Materialisation and all aforementioned incremental algorithms compute rule
consequences using \emph{semina\"ive evaluation}
\cite{abiteboul1995foundations}. The main benefit of this approach is that each
applicable inference is performed exactly once. However, all consequences of
certain rules or rule combinations can actually be computed without considering
every applicable inference. For example, consider applying a program that
axiomatises a relation $R$ as symmetric and transitive to input facts that
describe a connected graph consisting of $n$ vertices. In
Section~\ref{sec:motivation} we show that computing all consequences using
semina\"ive evaluation involves $O(n^3)$ rule applications, whereas a custom
algorithm can achieve the same goal using only $O(n^2)$ steps. Since
incremental maintenance algorithms are based on the semina\"ive algorithm, they
can suffer from similar deficiencies.

Approaches that can maintain the closure of specific datalog programs have
already been considered in the literature. For example, maintaining transitive
closure of a graph has been studied extensively
\cite{ibaraki1983line,la1987maintenance,king1999fully,demetrescu2000fully}.
Subercaze et al.\ (\citeyear{DBLP:journals/pvldb/SubercazeGCL16}) presented an
algorithm for the materialisation of the transitive and symmetric properties in
RDFS-Plus. Dong, Su, and Topor (\citeyear{DBLP:journals/amai/DongST95}) showed
that insertions into a transitively closed relation can be maintained by
evaluating four nonrecursive first-order queries. However, these approaches can
only handle datalog programs for which they have been specifically
developed---that is, the programs are not allowed to contain any additional
rules. The presence of other rules introduces additional complexity since
updates computed by specialised algorithms must be propagated to the remaining
rules and vice versa. Moreover, many of these approaches cannot handle deletion
of input facts, which is a key problem in incremental reasoning. Thus, it is
currently not clear whether and how customised algorithms can be used in
general-purpose datalog systems that must handle arbitrary datalog rules and
support incremental additions and deletions.

To address these issues, in this paper we present a modular framework for
materialisation computation and incremental materialisation maintenance that
can integrate specialised reasoning algorithms with the semina\"ive evaluation.
The framework partitions the rules of a datalog program into disjoint subsets
called \emph{modules}. For each module, four pluggable functions are used to
compute certain consequences of the module's rules; there are no restrictions
on how these functions are implemented, as long as their outputs satisfy
certain conditions. Moreover, if no specialised algorithm for a module is
available, the four functions can be implemented using semina\"ive evaluation.
Thus, our framework can efficiently handle certain combinations of rules, but
it can also handle arbitrary rules while avoiding repeated inferences.

We then examine a module that axiomatises the transitive closure, and a module
that axiomatises the symmetric--transitive closure. These modules capture node
reachability in directed and undirected graphs, respectively, both of which
frequently occur in practice and are thus highly relevant. We present the
functions necessary to integrate these modules into our framework and show that
they satisfy the properties needed for correctness. We also discuss the kinds
of input that are likely to benefit from modular reasoning.

We have implemented our algorithms and compared them on several real-life and
synthetic datasets. Our experiments illustrate the potential benefits of the
proposed solution: our approach often outperforms state-of-the-art algorithms,
sometimes by orders of magnitude. Our system and test data are available
online.\footnote{\url{http://krr-nas.cs.ox.ac.uk/2018/modular/}} All proofs of
our results are given in \iftoggle{withappendix}{the appendix.}{a technical
report~\cite{modularextendedversion}.}

\section{Preliminaries}\label{sec:preliminaries}

We now introduce datalog with stratified negation. A \emph{term} is a constant
or a variable. An \emph{atom} has the form ${P(t_1,\dots,t_k)}$, where $P$ is a
$k$-ary \emph{predicate} with ${k \geq 0}$, and each $t_i$, ${1 \leq i \leq
k}$, is a term. A \emph{fact} is a variable-free atom, and a \emph{dataset} is
a finite set of facts. A rule $r$ has the form
\begin{displaymath}
    B_1 \wedge \dots \wedge B_m \wedge \stnot B_{m+1} \wedge \dots \wedge \stnot B_n \rightarrow H,
\end{displaymath}
where ${0 \leq m \leq n}$, and $B_i$ and $H$ are atoms. For $r$ a rule,
${\head{r} = H}$ is the \emph{head}, ${\pbody{r} = \{ B_1, \dots, B_m \}}$ is
the set of \emph{positive body atoms}, and ${\nbody{r} = \{ B_{m+1}, \dots, B_n
\}}$ is the set of \emph{negative body atoms}. Each rule $r$ must be
\emph{safe}---that is, each variable occurring in $r$ must occur in at least
one positive body atom. A \emph{program} is a finite set of rules.

A \emph{stratification} $\lambda$ of a program $\Pi$ maps each predicate
occurring in $\Pi$ to a positive integer such that, for each rule ${r \in \Pi}$
with predicate $P$ in its head, ${\lambda(P) \geq \lambda(R)}$ (resp.\
${\lambda(P) > \lambda(R)}$) holds for each predicate $R$ occurring in
$\pbody{r}$ (resp.\ $\nbody{r}$). Such $r$ is \emph{recursive} w.r.t.\
$\lambda$ if ${\lambda(P) = \lambda(R)}$ holds for some predicate $R$ occurring
in $\pbody{r}$; otherwise, $r$ is \emph{nonrecursive} w.r.t.\ $\lambda$.
Program $\Pi$ is \emph{stratifiable} if a stratification $\lambda$ of $\Pi$
exists. For $s$ an integer, the \emph{stratum} $s$ of $\Pi$ is the program
$\strat{\Pi}{s}$ containing each rule ${r \in \Pi}$ whose head predicate $P$
satisfies ${\lambda(P) = s}$. Moreover, let $\rstrat{\Pi}{s}$ and
$\nrstrat{\Pi}{s}$ be the recursive and the nonrecursive subsets, respectively,
of $\strat{\Pi}{s}$. Finally, let ${\Out{s} = \{ P(c_1,\dots,c_n) \mid
\lambda(P) = s \text{ and } c_i \text{ are constants} \}}$.

A \emph{substitution} $\sigma$ is a mapping of finitely many variables to
constants. For $\alpha$ a term, an atom, a rule, or a set thereof,
$\alpha\sigma$ is the result of replacing each occurrence of a variable $x$ in
$\alpha$ with $\sigma(x)$, provided that the latter is defined.

If $r$ is a rule and $\sigma$ is a substitution mapping all variables of $r$ to
constants, then rule $r\sigma$ is an \emph{instance} of $r$. For $I$ a dataset,
we define the set $\apply{\Pi}{I}$ of all facts obtained by applying a program
$\Pi$ to $I$ as
\begin{displaymath}
    \apply{\Pi}{I} = \bigcup_{r \in \Pi} \{ \head{r\sigma} \mid \pbody{r\sigma} \subseteq I \text{ and } \nbody{r\sigma} \cap I = \emptyset \}.
\end{displaymath}
Let $E$ be a dataset (called \emph{explicit facts}) and let $\lambda$ be a
stratification of $\Pi$ with maximum stratum index $S$. Then, let
${I^{0}_{\infty} = E}$; for each ${s \geq 1}$, let ${I^s_0 =
I^{s-1}_{\infty}}$, let
\begin{displaymath}
    I^s_i = I^s_{i-1} \cup \apply{\Pi^s}{I^{s}_{i-1}} \text{ for } i > 0, \text{ and let } I^s_{\infty} = \bigcup_{i \geq 0}{I^s_i}.
\end{displaymath}
Set $I^S_{\infty}$ is the \emph{materialisation} of $\Pi$ w.r.t.\ $E$ and
$\lambda$. It is known that $I^S_{\infty}$ does not depend on $\lambda$, so we
write it as $\mat{\Pi}{E}$.

\section{Motivation}\label{sec:motivation}

In this section we show how custom algorithms can handle certain rule
combinations much more efficiently than semina\"ive evaluation. We consider
here only materialisation, but similar observations apply to incremental
maintenance algorithms as most of them use variants of semina\"ive evaluation.

\subsection{Semina\"ive Evaluation}

The semina\"ive algorithm \cite{abiteboul1995foundations} takes as input a set
of explicit facts $E$, a program $\Pi$, and a stratification $\lambda$ of
$\Pi$, and it computes $\mat{\Pi}{E}$. To apply each rule instance at most
once, in each round of rule application it identifies the `newly applicable'
rule instances (i.e., instances that depend on a fact derived in the previous
round) as shown in Algorithm~\ref{alg:seminaive}. For each stratum, the
algorithm initialises $\Delta$, the set of newly derived facts, by combining
the explicit facts in the current stratum (${E \cap \Out{s}}$) with the facts
derivable from previous strata via nonrecursive rules
($\apply{\nrstrat{\Pi}{s}}{I}$). Then, in
lines~\ref{alg:seminaive:innerloopstarts}--\ref{alg:seminaive:innerloopends} it
iteratively computes all consequences of $\Delta$. To this end, in
line~\ref{alg:seminaive:rec} it uses operator $\apply{\Pi}{I \appargs \Delta}$,
which extends $\apply{\Pi}{I}$ to allow identifying `newly applicable' rule
instances. Specifically, given datasets $I$ and ${\Delta \subseteq I}$,
operator $\apply{\Pi}{I \appargs \Delta}$ returns a set containing
$\head{r\sigma}$ for each rule ${r \in \Pi}$ and substitution $\sigma$ such
that ${\pbody{r\sigma} \subseteq I}$ and ${\nbody{r\sigma} \cap I = \emptyset}$
hold (i.e., rule instance $r\sigma$ is applicable to $I$), but also
${\pbody{r\sigma} \cap \Delta \neq \emptyset}$ holds (i.e., a positive body
atom of $r\sigma$ occurs in the set of facts $\Delta$ derived in the previous
round of rule application). It is not hard to see that the algorithm computes
${I = \mat{\Pi}{E}}$, and that it considers each rule instance $r\sigma$ at
most once.

\begin{algorithm}[t]
\caption{\textsc{Mat}$(\Pi, \lambda, E)$}\label{alg:seminaive}
\begin{algorithmiccont}
    \State $I \defeq \emptyset$
    \For{\textbf{each} stratum index $s$ with $1 \leq s \leq S$}
        \State $\Delta \defeq (E \cap \Out{s}) \cup \apply{\nrstrat{\Pi}{s}}{I}$            \label{alg:seminaive:updatedelta}
        \While{$\Delta \neq \emptyset$}                                                     \label{alg:seminaive:innerloopstarts}
            \State $I \defeq I \cup \Delta$ 
            \State $\Delta \defeq \apply{\rstrat{\Pi}{s}}{I \appargs \Delta} \setminus I$   \label{alg:seminaive:rec}
        \EndWhile                                                                           \label{alg:seminaive:innerloopends}
    \EndFor                                                                                 \label{alg:seminaive:looponsends}
\end{algorithmiccont}
\end{algorithm}

\subsection{Problems with the Semina\"ive Evaluation}

Although semina\"ive evaluation does not repeat derivations, it always
considers each applicable rule instance. However, facts are often derived via
multiple, distinct rule instances; this is particularly common with recursive
rules, but it can also occur with nonrecursive rules only. We are unaware of a
general technique that can prevent such derivations. We next present two
programs for which materialisation can be computed without considering all
applicable rule instances, thus showing how semina\"ive evaluation can be
suboptimal.

\begin{example}\label{ex:trans}
Let $\Pi$ be the program containing rule \eqref{transitiverule} and let ${E =
\{ R(c_i,c_{i+1}) \mid 0 \leq i \leq n \}}$.
\begin{align}
    R(x,y) \wedge R(y,z) \rightarrow R(x,z)             \label{transitiverule}
\end{align}
Clearly, ${I = \mat{\Pi}{E} = \{ R(c_i,c_j) \mid 0 \leq i < j \leq n \}}$, so
each rule instance of the form
\begin{align}
    R(c_i,c_j) \wedge R(c_j,c_k) \rightarrow R(c_i,c_k) \label{transitiveruleinst}
\end{align}
with ${1 \leq i < j < k \leq n}$ is applicable to $I$.
Algorithm~\ref{alg:seminaive} considers all of these $O(n^3)$ rule instances.

We next present an outline of an approach that is still cubic in general, but
on this specific input runs in $O(n^2)$ time. The key is to distinguish the set
$X$ of `external' facts given to $\Pi$ as input from the `internal' facts
derived by $\Pi$. We can transitively close $R$ by iteratively considering
pairs of facts ${R(u,v) \in X}$ and ${R(v,w)}$. That is, we require the first
fact to be in $X$, but place no restriction on the second fact. (We could have
equivalently required the second fact to be in $X$.) In our example, we have
${X = E}$, so the algorithm considers only rule instances of the form
\begin{align}
    R(c_i,c_{i+1}) \wedge R(c_{i+1},c_k) \rightarrow R(c_i,c_k)
\end{align}
for ${0 \leq i < k \leq n}$, of which there are $O(n^2)$ many. Intuitively,
this is analogous to replacing the predicate $R$ in all explicit facts with
$X$, and using a linear rule
\begin{align}
    X(x,y) \wedge R(y,z) \rightarrow R(x,z)             \label{lintransitiverule}
\end{align}
instead of rule \eqref{transitiverule}. In our approach, however, other rules
can derive $R$-facts so the set $X$ is not fixed; thus, rule
\eqref{transitiverule} cannot be simply replaced with
\eqref{lintransitiverule}. Our approach `simulates' such linearisation, and it
can be expected to perform well whenever the other rules derive fewer facts
than rule \eqref{transitiverule}.
\end{example}

\begin{example}\label{ex:sym-trans}
Let $\Pi$ consist of rules \eqref{transitiverule} and \eqref{symrule}, and let
${E = \{ R (c_i,c_{i + 1}) \mid 1 \leq i < n \} \cup \{ R(c_n,c_1) \}}$.
\begin{align}
    R(x,y) \rightarrow R(y,x)                           \label{symrule}
\end{align}
Now ${I = \mat{\Pi}{E} = \{ R(c_i,c_j) \mid 1 \leq i, j \leq n \}}$, so each
instance of the form \eqref{transitiveruleinst} with ${1 \leq i,j,k \leq n}$ is
applicable to $I$. Algorithm~\ref{alg:seminaive} considers all of these
$O(n^3)$ rule instances.

However, we can view any relation $R$ as an undirected graph with $n$ vertices.
To compute the symmetric--transitive closure of $R$, we first compute the
connected components of $R$, and, for each connected component $C$, we
enumerate all ${u,v \in C}$ and derive $R(u,v)$. The first step is linear in
the size of $R$ and the second step requires $O(n^2)$ time, so the algorithm
runs in $O(n^2)$ time on any $R$.
\end{example}

\section{Framework}\label{sec:framework}

In this section we present a general framework for materialisation and
incremental reasoning that can avoid the deficiencies outlined in
Section~\ref{sec:motivation} for certain rule combinations. Our framework
focuses on recursive rules only: nonrecursive rules $\nrstrat{\Pi}{s}$ are
evaluated just once in each stratum, which is usually efficient. In contrast,
the recursive part $\rstrat{\Pi}{s}$ of each stratum $\strat{\Pi}{s}$ must be
evaluated iteratively, which is a common source of inefficiency. Thus, our
framework splits $\rstrat{\Pi}{s}$ into $n(s)$ mutually disjoint, nonempty
programs $\rstrat{\Pi}{s,i}$, ${1 \leq i \leq n(s)}$, called \emph{modules}.
(We let ${n(s) = 0}$ if ${\rstrat{\Pi}{s} = \emptyset}$.) Our notion of modules
should not be confused with ontology modules: the latter are subsets of an
ontology that are semantically independent from each other in a well-defined
way, whereas our modules are just arbitrary program subsets. Each module is
handled using `plugin' functions that compute certain consequences of
$\rstrat{\Pi}{s,i}$. These functions can be implemented as desired, as long as
their results satisfy certain properties that guarantee correctness. We present
our framework in two steps: in Section~\ref{sec:framework:materialisation} we
consider materialisation, and in Section~\ref{sec:framework:incremental} we
focus on incremental reasoning. Then, in Sections~\ref{sec:trans}
and~\ref{sec:sym-trans} we discuss how to realise these `plugin' functions for
certain common modules.

Before proceeding, we generalise operator $\apply{\Pi}{I \appargs \Delta}$ as
follows. Given datasets $\ipos$, $\ineg$, $\Deltapos$, and $\Deltaneg$ where
${\Deltapos \subseteq \ipos}$ and ${\Deltaneg \cap \ineg = \emptyset}$, let
\begin{displaymath}
    \begin{array}{@{}l@{}}
        \apply{\Pi}{\ipos, \ineg \appargs \Deltapos, \Deltaneg} = \bigcup_{r \in \Pi} \{ \head{r\sigma} \; \mid \\[1ex]
        \hspace{2.2cm} \pbody{r\sigma} \subseteq \ipos \text{ and } \nbody{r\sigma} \cap \ineg = \emptyset, \text{ and} \\[0.3ex]
        \hspace{2.2cm} \pbody{r\sigma} \cap \Deltapos \neq \emptyset \text{ or } \nbody{r\sigma} \cap \Deltaneg \neq \emptyset \}.
    \end{array}
\end{displaymath}
When the condition in the last line is not required, we simply write
$\apply{\Pi}{\ipos,\ineg}$. Moreover, we omit $\ineg$ when ${\ipos = \ineg}$,
and we omit $\Deltaneg$ when ${\Deltaneg = \emptyset}$. Intuitively, this
operator computes the consequences of $\Pi$ by evaluating the positive and the
negative body atoms in $\ipos$ and $\ineg$, respectively, while ensuring in
each derivation that either a positive or a negative body atom is true in
$\Deltapos$ or $\Deltaneg$, respectively. Our incremental algorithm uses this
operator to identify the consequences of $\Pi$ that are affected by the changes
to the facts matching the positive and the negative body atoms of the rules in
$\Pi$. For example, if the facts in $\Deltapos$ are added to (resp.\ removed
from) the materialisation, then $\apply{\Pi}{\ipos \appargs \Deltapos}$
contains the consequences of the rule instances that start (resp.\ cease) to be
applicable because a positive body atom matches to a fact in $\Deltapos$. The
set $\Deltaneg$ is used to analogously capture the consequences of the negative
body atoms of the rules in $\Pi$.

\subsection{Computing the Materialisation}\label{sec:framework:materialisation}

Our modular materialisation algorithm uses a `plugin' function $\addfn{s,i}$
for each module $\rstrat{\Pi}{s,i}$. The function takes as arguments datasets
$\ipos$, $\ineg$, and $\Delta$ such that ${\Delta \subseteq \ipos}$, and it
closes $\ipos$ with all consequences of $\rstrat{\Pi}{s,i}$ that depend on
$\Delta$. Each invocation of these functions must satisfy the following
properties in order to guarantee correctness of our algorithm.

\begin{definition}\label{def:add}
    Function $\addfn{}$ \emph{captures} a datalog program $\Pi$ on datasets
    $\ipos$, $\ineg$, and $\Delta$ with ${\Delta \subseteq \ipos}$ if the
    result of $\add{}{\ipos}{\ineg}{\Delta}$ is the smallest dataset $J$ that
    satisfies ${\apply{\Pi}{\ipos \cup J, \ineg \appargs \Delta \cup J}
    \subseteq \ipos \cup J}$.
\end{definition}

For brevity, in the rest of the paper we often say just `$\addfn{}$ captures
$\Pi$' without specifying the datasets whenever the latter are clear from the
context. In the absence of a customised algorithm, $\addfn{}$ can always be
realised using the semina{\"i}ve evaluation strategy as follows:
\begin{itemize}
    \item let ${\Delta_0 = \Delta}$ and $J_0 = \emptyset$,

    \item for $i$ starting with 0 onwards, if ${\Delta_i = \emptyset}$, stop
    and return $J_i$; otherwise, let ${\Delta_{i+1} = \apply{\Pi}{\ipos \cup
    J_i, \ineg \appargs \Delta_i} \setminus (\ipos \cup J_i)}$ and ${J_{i+1} =
    J_i \cup \Delta_{i+1}}$ and proceed to $i+1$.
\end{itemize}
However, a custom implementation of $\addfn{}$ will typically not examine all
rule instances from the above computation in order to optimise reasoning with
certain modules.

Algorithm~\ref{and:matmod} formalises our modular approach to datalog
materialisation. It takes as input a program $\Pi$, a stratification $\lambda$
of $\Pi$, and a set of explicit facts $E$, and it computes $\mat{\Pi}{E}$. The
algorithm's structure is similar to Algorithm~\ref{alg:seminaive}. For each
stratum of $\Pi$, both algorithms first apply the nonrecursive rules, and then
they apply the recursive rules iteratively up to a fixpoint. The main
difference is that, given a set of facts $\Delta$ derived from the previous
iteration, Algorithm~\ref{and:matmod} computes the consequences of $\Delta$ for
each module independently using $\addfn{s,i}$
(line~\ref{and:matmod:updatedeltai}); note that each $\Delta_i$ is closed under
$\rstrat{\Pi}{s,i}$, which is key to the performance of our approach. The
algorithm then combines the consequences of all modules
(line~\ref{and:matmod:combinedelta}) before proceeding to the next iteration.

If each $\addfn{s,i}$ function is implemented using semina\"ive evaluation as
described earlier, then the algorithm does not consider a rule instance more
than once. This is achieved by passing ${\Delta \setminus \Delta_i}$ to
$\addfn{s,i}$ in line~\ref{and:matmod:updatedeltai}: only facts derived by
other modules in the previous iteration are considered `new' for $\addfn{s,i}$,
which is possible since the facts in $\Delta_i$ have been produced by the $i$th
module in the previous iteration. Theorem~\ref{theorem:correctness-mat}
captures these properties formally.

\begin{restatable}{theorem}{correctnessmat}\label{theorem:correctness-mat}
    Algorithm~\ref{and:matmod} computes $I$ as $\mat{\Pi}{E}$ if function
    $\addfn{s,i}$ captures $\rstrat{\Pi}{s,i}$ in each of its calls. Moreover,
    if all $\addfn{s,i}$ use the semina{\"i}ve strategy, each applicable rule
    instance is considered at most once.
\end{restatable}

\begin{algorithm}[t]
\caption{\textsc{Mat-Mod}$(\Pi, \lambda, E)$}\label{and:matmod}
\begin{algorithmiccont}
    \State $I \defeq \emptyset$
    \For{\textbf{each} stratum index $s$ with $1 \leq s \leq S$}                    \label{and:matmod:looponsstarts}
        \State $\Delta_1 \defeq \dots \defeq \Delta_{n(s)} \defeq \emptyset$        \label{and:matmod:init}
        \State $\Delta \defeq (E \cap \Out{s}) \cup \apply{\nrstrat{\Pi}{s}}{I}$    \label{and:matmod:updatedelta}
        \While{$\Delta \neq \emptyset$}
            \State $I \defeq I \cup \Delta$                                         \label{and:matmod:innerloopstarts}
            \For{\textbf{each} $i$ with $1 \leq i \leq n(s)$}
                \State $\Delta_i \defeq \add{s,i}{I}{I}{\Delta \setminus \Delta_i}$ \label{and:matmod:updatedeltai}
            \EndFor
            \State $\Delta \defeq \Delta_1 \cup \dots \cup \Delta_{n(s)}$           \label{and:matmod:combinedelta}
        \EndWhile                                                                   \label{and:matmod:innerloopends}
    \EndFor \label{and:matmod:looponsends}
\end{algorithmiccont}
\end{algorithm}

\subsection{Incremental Updates}\label{sec:framework:incremental}

Our modular incremental materialisation maintenance algorithm is based on the
DRed$^c$ algorithm by \citeauthor{DBLP:conf/aaai/HuMH18}
(\citeyear{DBLP:conf/aaai/HuMH18}), which is a variant of the well-known DRed
algorithm \cite{gupta1993maintaining}. For each fact, DRed$^c$ maintains two
counters that track the number of nonrecursive and recursive derivations of the
fact. The algorithm proceeds in three steps. During the deletion phase,
DRed$^c$ iteratively computes the consequences of the deleted facts, similar to
DRed, while adjusting the counters accordingly. However, to optimise
overdeletion, deletion propagation stops on facts with a nonzero nonrecursive
counter. In the one-step rederivation phase, DRed$^c$ identifies the facts that
were overdeleted but can be rederived from the remaining facts in one step by
simply checking the recursive counters: if the counter is nonzero, then the
corresponding fact is rederived. In the insertion phase, DRed$^c$ computes the
consequences of the rederived and the inserted facts using semina\"ive
evaluation, which we have already discussed.

Our modular incremental algorithm handles nonrecursive rules in the same way as
DRed$^c$. Thus, the nonrecursive counters, which record the number of
nonrecursive derivations of each fact, can be maintained globally just as in
DRed$^c$. In contrast, as discussed in Section~\ref{sec:motivation}, custom
algorithms for recursive modules will usually not consider all applicable rule
instances, so counters of recursive derivations cannot be maintained globally.
Nevertheless, certain modules can maintain recursive counters internally (e.g.,
the module based on the semina\"ive evaluation can do so).

In addition to function $\addfn{s,i}$ from
Section~\ref{sec:framework:materialisation}, our modular incremental reasoning
algorithm uses three further functions: $\difffn{s,i}$, $\delfn{s,i}$, and
$\redfn{s,i}$. Definition~\ref{def:diff} captures the requirements on
$\difffn{s,i}$. Intuitively, $\diff{s,i}{\ipos}{\Deltapos}{\Deltaneg}$
identifies the consequences of $\rstrat{\Pi}{s,i}$ affected by the addition of
the facts in $\Deltapos$ and removal of the the facts $\Deltaneg$,
respectively, with both sets containing facts from earlier strata.

\begin{definition}\label{def:diff}
    Function $\difffn{}$ \emph{captures} a datalog program $\Pi$ on datasets
    $\ipos$, $\Deltapos$, and $\Deltaneg$ where ${\Deltapos \subseteq \ipos}$,
    ${\Deltaneg \cap \ipos = \emptyset}$, and both $\Deltapos$ and $\Deltaneg$
    do not contain predicates occurring in rule heads in $\Pi$ if
    $\diff{}{\ipos}{\Deltapos}{\Deltaneg} = \apply{\Pi}{\ipos \appargs
    \Deltapos, \Deltaneg}$.
\end{definition}

Function $\delfn{s,i}$ captures overdeletion: if the facts in $\Delta$ are
deleted, then ${\del{s,i}{\ipos}{\ineg}{\Delta}{\Cnr}}$ returns the
consequences of $\rstrat{\Pi}{s,i}$ that must be overdeleted as well. The
function can use the nonrecursive counters $\Cnr$ in order to stop overdeletion
as in DRed$^c$. We do not specify exactly what the functions must return: as we
discuss in Section~\ref{sec:sym-trans}, computing the smallest set that needs
to be overdeleted might require considering all rule instances as in DRed$^c$,
which would miss the point of modular reasoning. Instead, we specify the
required output in terms of a lower bound $J_l$ and an upper bound $J_u$.
Intuitively, $J_l$ and $J_u$ contain facts that would be overdeleted in
DRed$^c$ and DRed, respectively.

\begin{definition}\label{def:del}
    Function $\delfn{}$ \emph{captures} a datalog program $\Pi$ on datasets
    $\ipos$, $\ineg$, $\Delta$ with ${\Delta \subseteq \ipos}$, and a mapping
    $\Cnr$ of facts to integers if ${J_l \subseteq
    \del{}{\ipos}{\ineg}{\Delta}{\Cnr} \subseteq J_u}$ where
    \begin{itemize}
        \item the \emph{lower bound} $J_l$ is the smallest dataset such that,
        for each ${F \in \apply{\Pi}{\ipos,\ineg \appargs \Delta \cup J_l}}$,
        either ${F \in \Delta \cup J_l}$ or ${\Cnr(F) > 0}$ holds, and

        \item the \emph{upper bound} $J_u$ is the smallest dataset that
        satisfies ${\apply{\Pi}{\ipos, \ineg \appargs \Delta \cup J_u}
        \subseteq \Delta \cup J_u}$.
    \end{itemize}
\end{definition}

Finally, function $\redfn{s,i}$ captures rederivation: if facts in $\Delta$ are
overdeleted, then $\red{s,i}{\ipos}{\ineg}{\Delta}$ returns all facts in
$\Delta$ that can be rederived from ${\ipos \setminus \Delta}$ and
$\rstrat{\Pi}{s,i}$ in \emph{one or more steps}. This is different from DRed
and DRed$^c$, which both perform only one-step rederivation. This change is
important in our framework because, as we shall see in
Section~\ref{sec:sym-trans}, $\redfn{s,i}$ provides the opportunity for a
module to adjust its internal data structures after deletion.

\begin{definition}\label{def:red}
    Function $\redfn{s,i}$ \emph{captures} a datalog program $\Pi$ on datasets
    $\ipos$, $\ineg$, $\Delta$ with ${\Delta \subseteq \ipos}$ if the result of
    $\red{}{\ipos}{\ineg}{\Delta}$ is the smallest dataset $J$ that satisfies
    ${\apply{\Pi}{(\ipos \setminus \Delta) \cup J, \ineg} \cap \Delta \subseteq
    J}$.
\end{definition}

Algorithm~\ref{alg:dredmod} formalises our modular approach to incremental
maintenance. The algorithm takes as input a program $\Pi$, a stratification
$\lambda$ of $\Pi$, a set of explicit facts $E$, the materialisation ${I =
\mat{\Pi}{E}}$, the sets of facts $E^-$ and $E^+$ to delete from and add to
$E$, and a map $\Cnr$ that records the number of nonrecursive derivations of
each fact. The algorithm updates $I$ to ${\mat{\Pi}{(E \setminus E^-) \cup
E^+}}$. We next describe the two main steps of the algorithm.

In the overdeletion phase, the algorithm first initialises the set of facts to
delete $\Delta$ as the union of the explicitly deleted facts $(E^- \cap
\Out{s})$ and the facts affected by changes in previous strata
(lines~\ref{alg:dredmod:delta-init} and~\ref{alg:dredmod:del:ND-init}). Then,
in lines~\ref{alg:dredmod:del:loopstarts}--\ref{alg:dredmod:del:loopends} the
algorithm computes all consequences of $\Delta$. In each iteration, function
$\delfn{s,i}$ is called for each module to identify the consequences of
$\rstrat{\Pi}{s,i}$ that must be overdeleted due to the deletion of $\Delta$
(line~\ref{alg:dredmod:del:updatedeltai}). As in Algorithm~\ref{and:matmod},
the third argument of $\delfn{s,i}$ is ${\Delta \setminus \Delta_i}$, which
guarantees that the function will not be applied to its own consequences.

In the second step, the algorithm first identifies the rederivable facts by
calling $\redfn{s,i}$ for each module
(lines~\ref{alg:dredmod:redinsert:recstarts}--\ref{alg:dredmod:redinsert:recends}).
Then, the consequences of the rederived facts, the explicitly added facts ($E^+
\cap \Out{s}$), and the facts added due to changes in previous strata are
computed in the loop of
lines~\ref{alg:dredmod:redinsert:insertloopstarts}--\ref{alg:dredmod:redinsert:insertloopends}
analogously to Algorithm~\ref{and:matmod}. Although $\redfn{s,i}$ rederives
facts in one or more steps as opposed to the one-step rederivation in DRed and
DRed$^c$, this extra effort is not repeated during insertion since
$\addfn{s,i}$ is not applied to the consequences of module $i$.
Theorem~\ref{theorem:correctness-dredmod} states that the algorithm is correct.

\begin{restatable}{theorem}{correctnessdredcmod}\label{theorem:correctness-dredmod}
    Algorithm~\ref{alg:dredmod} updates $I$ from $\mat{\Pi}{E}$ to
    ${\mat{\Pi}{(E \setminus E^-) \cup E^+}}$ if functions $\addfn{s,i}$,
    $\delfn{s,i}$, $\difffn{s,i}$, and $\redfn{s,i}$ capture
    $\rstrat{\Pi}{s,i}$ in all of their calls.
\end{restatable}

\begin{algorithm}[t]
\caption{\textsc{DRed$^c$-Mod}$(\Pi, \lambda, E, I, E^-, E^+, \Cnr)$}\label{alg:dredmod}
\begin{algorithmiccont}
    \State $D \defeq A \defeq \emptyset$, \quad $E^- = (E^- \cap E) \setminus E^+$, \quad $E^+ = E^+ \setminus E$                               \label{alg:dredmod:init}
    \For{\textbf{each} stratum index $s$ with ${1 \leq s \leq S}$}                                                                              \label{alg:dredmod:stratum-loop}
        \State \Call{Overdelete}{}                                                                                                              \label{alg:dredmod:Overdelete}
        \State \Call{Rederive-Insert}{}                                                                                                         \label{alg:dredmod:Insert}
    \EndFor                                                                                                                                     \label{alg:dredmod:stratum-loop:end}
    \State $E \defeq (E \setminus E^-) \cup E^+$, \quad $I \defeq (I \setminus D) \cup A$                                                       \label{alg:dredmod:update-I}
    \vspace{0.2cm}
    \Procedure{Overdelete}{}
        \State $\Delta_1 \defeq \dots \defeq \Delta_{n(s)} \defeq \emptyset$                                                                    \label{alg:dredmod:delta-init}
        \State $\Delta \defeq (E^- \cap \Out{s}) \cup \apply{\nrstrat{\Pi}{s}}{I \appargs D \setminus A, A \setminus D}$ and update $\Cnr$      \label{alg:dredmod:del:ND-init}
        \For{\textbf{each} $i$ with $1 \leq i \leq n(s)$}
            \State $\Delta \defeq \Delta \cup \diff{s,i}{I}{D \setminus A}{A \setminus D}$                                                      \label{alg:dredmod:reddel:diff}
        \EndFor
        \While{$\Delta \neq \emptyset$}                                                                                                         \label{alg:dredmod:del:loopstarts}
            \For{\textbf{each} $i$ with $1 \leq i \leq n(s)$}
                \State $\Delta_i \defeq \del{s,i}{I \setminus (D \setminus A)}{I \cup A}{\Delta \setminus \Delta_i}{\Cnr}$                      \label{alg:dredmod:del:updatedeltai}
            \EndFor
            \State $D \defeq D \cup \Delta$
            \State $\Delta \defeq \Delta_1 \cup \dots \cup \Delta_{n(s)}$
        \EndWhile                                                                                                                               \label{alg:dredmod:del:loopends}
    \EndProcedure
    \vspace{0.2cm}
    \Procedure{Rederive-Insert}{}
        \State $\Delta \defeq (E^+ \cap \Out{s}) \cup \apply{\nrstrat{\Pi}{s}}{(I \setminus D) \cup A \appargs A \setminus D, D \setminus A}$
        \Statex \hspace{1.5cm} and update $\Cnr$
        \For{\textbf{each} $i$ with $1 \leq i \leq n(s)$}                                                                                       \label{alg:dredmod:redinsert:recstarts}
            \State $\Delta_i \defeq \red{s,i}{I}{(I \setminus D) \cup A}{D \setminus A}$                                                        \label{alg:dredmod:redinsert:red}
            \State $\Delta \defeq \Delta \cup \Delta_i \cup \diff{s,i}{(I \setminus D) \cup A}{A \setminus D}{D \setminus A}$                   \label{alg:dredmod:redinsert:diff}
        \EndFor                                                                                                                                 \label{alg:dredmod:redinsert:recends}
        \While{$\Delta \neq \emptyset$}                                                                                                         \label{alg:dredmod:redinsert:insertloopstarts}
            \State $A \defeq A \cup \Delta$
            \For{\textbf{each} $i$ with $1 \leq i \leq n(s)$}
                \State $\Delta_i \defeq \add{s,i}{(I \setminus D) \cup A}{(I \setminus D) \cup A}{\Delta \setminus \Delta_i}$                   \label{alg:dredmod:redinsert:add}
            \EndFor
            \State $\Delta \defeq \Delta_1 \cup \dots \cup \Delta_{n(s)}$
        \EndWhile                                                                                                                               \label{alg:dredmod:redinsert:insertloopends}
    \EndProcedure
\end{algorithmiccont}
\end{algorithm}

\section{Transitive Closure}\label{sec:trans}

We now consider a module consisting of a single rule \eqref{transitiverule}
axiomatising a relation $R$ as transitive. Following the ideas from
Example~\ref{ex:trans}, we distinguish the `internal' facts produced by rule
\eqref{transitiverule} from the `external' facts produced by other rules. We
keep track of the latter in a global set $X_R$ that is initialised to the empty
set. A key invariant of our approach is that each fact $R(a_0,a_n)$ is produced
by a chain ${\{ R(a_0, a_1), \dots, R(a_{n-1}, a_n) \} \subseteq X_R}$ of
`external' facts. Thus, we can transitively close $R$ by considering pairs of
$R$-facts where at least one of them is contained in $X_R$, which can greatly
reduce the number of inferences. A similar effect could be achieved by
rewriting the input program: we introduce a fresh predicate $X_R$, and we
replace by $X_R$ each occurrence of $R$ in the head of a rule, as well as one
of the two occurrences of $R$ in the body of rule \eqref{transitiverule}. Such
an approach, however, introduces the facts containing the auxiliary predicate
$X_R$ into the materialisation and thus reveals implementation details to the
users. Moreover, the rederivation step can be realised very efficiently in our
approach.

Based on the above idea, function $\addfn{\mathsf{tc}(R)}$, shown in
Algorithm~\ref{alg:addtc}, essentially implements semina\"ive evaluation for
rule \eqref{lintransitiverule}: the loops in lines
\ref{alg:addtc:loop1starts}--\ref{alg:addtc:loop1ends} and
\ref{alg:addtc:loops2starts}--\ref{alg:addtc:loop2ends} handle the two delta
rules derived from \eqref{lintransitiverule}. For $\difffn{\mathsf{tc}(R)}$,
note that sets ${A \setminus D}$ and ${D \setminus A}$ in
lines~\ref{alg:dredmod:reddel:diff} and~\ref{alg:dredmod:redinsert:diff} of
Algorithm~\ref{alg:dredmod} always contain facts with predicates that do not
occur in $\rstrat{\Pi}{s,i}$ in rule heads; thus, since $R$ occurs in the head
of rule \eqref{transitiverule}, these sets contain facts whose predicate is
different from $R$, so $\difffn{\mathsf{tc}(R)}$ can simply return the empty
set. Function $\delfn{\mathsf{tc}(R)}$, shown in Algorithm~\ref{alg:deltc},
implements semina\"ive evaluation for rule \eqref{lintransitiverule}
analogously to $\addfn{\mathsf{tc}(R)}$. The main difference is that only facts
whose nonrecursive counter is zero are overdeleted, which mimics overdeletion
in DRed$^c$. As a result, not all facts processed in
lines~\ref{alg:deltc:loop1starts} and~\ref{alg:deltc:loop2starts} are added to
$J$ so, to avoid repeatedly considering such facts, the algorithm maintains the
set $S$ of `seen' facts. Finally, function $\redfn{\mathsf{tc}(R)}$, shown in
Algorithm~\ref{alg:redtc}, identifies for each source vertex $u$ all vertices
reachable by the external facts in $X_R$.

\begin{restatable}{theorem}{correctnesstrans}\label{theorem:correctness-trans}
    In each call in Algorithms~\ref{and:matmod} and~\ref{alg:dredmod},
    functions $\addfn{\mathsf{tc}(R)}$, $\delfn{\mathsf{tc}(R)}$,
    $\difffn{\mathsf{tc}(R)}$, and $\redfn{\mathsf{tc}(R)}$ capture a datalog
    program that axiomatises relation $R$ as transitive.
\end{restatable}

\begin{table}[t]
\begin{algorithm}[H]
\caption{$\add{\mathsf{tc}(R)}{\ipos}{\ineg}{\Delta}$}\label{alg:addtc}
\begin{algorithmiccont}
    \State $J \defeq \emptyset$, \quad $Q \defeq \Delta$, \quad $X_R \defeq X_R \cup \Delta$                \label{alg:addtc:init}
    \For{\textbf{each} $R(u,v) \in \Delta$ and \textbf{each} $R(v,w) \in \ipos \setminus \Delta$}           \label{alg:addtc:loop1starts}
        \State \textbf{add} $R(u,w)$ to $Q$ and $J$                                                         \label{alg:addtc:addtoq}
    \EndFor                                                                                                 \label{alg:addtc:loop1ends}
    \While{$Q \neq \emptyset$}                                                                              \label{alg:addtc:loops2starts}
        \State \textbf{remove} an arbitrarily chosen fact $R(v,w)$ from $Q$
        \For{\textbf{each} $R(u,v) \in X_R$ such that $R(u,w) \not\in \ipos \cup J$}                        \label{alg:addtc:prune}
            \State \textbf{add} $R(u,w)$ to $Q$ and $J$                                                     \label{alg:addtc:addtoqandj}
        \EndFor
    \EndWhile                                                                                               \label{alg:addtc:loop2ends}
    \State \Return $J$
\end{algorithmiccont}
\end{algorithm}
\vspace{-0.9cm}
\begin{algorithm}[H]
\caption{$\del{\mathsf{tc}(R)}{\ipos}{\ineg}{\Delta}{\Cnr}$}\label{alg:deltc}
\begin{algorithmiccont}
    \State $J \defeq \emptyset$, \quad $Q \defeq S \defeq \Delta$, \quad $X_R \defeq X_R \setminus \Delta$
    \For{\textbf{each} $R(u,v) \in \Delta$ and \textbf{each} $R(v,w) \in \ipos \setminus S$}                \label{alg:deltc:loop1starts}
        \State \textbf{add} $R(u,w)$ to $Q$ and $S$
        \If{$\Cnr(R(u,w)) = 0$}
            \textbf{add} $R(u,w)$ to $J$
        \EndIf
    \EndFor                                                                                                 \label{alg:deltc:loop1ends}
    \While{$Q \neq \emptyset$}                                                                              \label{alg:deltc:loop2starts}
        \State \textbf{remove} an arbitrarily chosen fact $R(v,w)$ from $Q$
        \For{\textbf{each} $R(u,v) \in X_R$ such that $R(u,w) \in \ipos \setminus S$}
            \State \textbf{add} $R(u,w)$ to $Q$ and $S$
            \If{$\Cnr(R(u,w)) = 0$}
                \textbf{add} $R(u,w)$ to $J$
            \EndIf
        \EndFor
    \EndWhile                                                                                               \label{alg:deltc:loop2ends}
    \State \Return $J \setminus \Delta$
\end{algorithmiccont}
\end{algorithm}
\vspace{-0.9cm}
\begin{algorithm}[H]
\caption{$\red{\mathsf{tc}(R)}{\ipos}{\ineg}{\Delta}$}\label{alg:redtc}
\begin{algorithmiccont}
    \State $J \defeq \emptyset$
    \For{\textbf{each} $u$ such that there exist $v$ with $R(u,v) \in \Delta$}                              \label{alg:deltc:findsource}
        \For{\textbf{each} $w$ reachable from $u$ via $R$ facts in $X_R$}
            \State \textbf{add} $R(u,w)$ to $J$                                                             \label{alg:deltc:addtoj}
        \EndFor
    \EndFor
    \State \Return $J \cap \Delta$
\end{algorithmiccont}
\end{algorithm}
\end{table}

\section{Symmetric--Transitive Closure}\label{sec:sym-trans}

We now consider a module consisting of two rules, \eqref{transitiverule} and
\eqref{symrule}, axiomatising a relation $R$ as transitive and symmetric. As in
Example~\ref{ex:sym-trans}, we can view relation $R$ as an undirected graph. To
compute the materialisation, we extract the set $C_R$ of connected
components---that is, each ${U \in C_R}$ is a set of mutually connected
vertices in the symmetric--transitive closure of $R$; finally, we derive
$R(u,v)$ for all $u$ and $v$ in each component ${U \in C_R}$. Set $C_R$ is
global and is initially empty.

Based on this idea, function $\addfn{\mathsf{stc}(R)}$, shown in
Algorithm~\ref{alg:addstc}, uses an auxiliary function \Call{closeEdges}{} to
incrementally update the set $C_R$ by processing each fact ${R(u,v) \in
\Delta}$ in lines \ref{alg:addstc:loopstarts}--\ref{alg:addstc:loopends}: if
either $u$ or $v$ does not occur in a component in $C_R$, then the respective
component is created in $C_R$ (lines~\ref{alg:addstc:createucomonent}
and~\ref{alg:addstc:createvcomonent}); and if $u$ and $v$ belong to distinct
components $U$ and $V$, then $U$ and $V$ are merged into a single component and
all $R$-facts connecting $U$ and $V$ are added (lines
\ref{alg:addstc:connectuvstarts}--\ref{alg:addstc:connectuvends}). For the same
reasons as in Section~\ref{sec:trans}, function $\difffn{\mathsf{tc}(R)}$ can
simply return the empty set. Function $\delfn{\mathsf{stc}(R)}$, shown in
Algorithm~\ref{alg:delstc}, simply overdeletes all facts $R(u',v')$ whose
nonrecursive counter is zero and where both $u'$ and $v'$ belong to a component
$U$ containing both vertices of a fact $R(u,v)$ in $\Delta$. Those facts
$R(u',v')$ for which the nonrecursive counter is nonzero will hold after
overdeletion, so they are kept in an initially empty global set $Y_R$ so that
they can be used for rederivation later. Finally, function
$\redfn{\mathsf{stc}(R)}$, shown in Algorithm~\ref{alg:redstc}, simply closes
the set $Y_R$ in the same way as during addition, and it empties the set $Y_R$.
While this creates a dependency between $\delfn{\mathsf{stc}(R)}$ and
$\redfn{\mathsf{stc}(R)}$, the order in which these functions are called in
Algorithm~\ref{alg:dredmod} ensures that the set $Y_R$ is maintained correctly.

\begin{restatable}{theorem}{correctnesssymtrans}\label{theorem:correctness-symtrans}
    In each call in Algorithms~\ref{and:matmod} and~\ref{alg:dredmod},
    functions $\addfn{\mathsf{stc}(R)}$, $\delfn{\mathsf{stc}(R)}$,
    $\difffn{\mathsf{stc}(R)}$, and $\redfn{\mathsf{stc}(R)}$ capture a datalog
    program that axiomatises $R$ as symmetric--transitive.
\end{restatable}

\begin{table}[tb]
\begin{algorithm}[H]
\caption{$\add{\mathsf{stc}(R)}{\ipos}{\ineg}{\Delta}$}\label{alg:addstc}
\begin{algorithmiccont}
    \State \Return $\textsc{CloseEdges}(\Delta) \setminus \ipos$
    \vspace{0.2cm}
    \Function{CloseEdges}{$\Delta$}
        \State $J \defeq \emptyset$
        \For{\textbf{each} $R(u,v) \in \Delta$}                                                     \label{alg:addstc:loopstarts}
            \If{no $U \in C_R$ exists such that $u \in U$}
                \State \textbf{add} $\{ u \}$ to $C_R$, and $R(u,u)$ to $J$                         \label{alg:addstc:createucomonent}
            \EndIf
            \If{no $V \in C_R$ exists such that $v \in V$}
                \State \textbf{add} $\{ v \}$ to $C_R$, and $R(v,v)$ to $J$                         \label{alg:addstc:createvcomonent}
            \EndIf                                                                                  \label{alg:addstc:vends}
            \If{$u$ and $v$ belong to distinct $U,V \in C_R$, resp.}                                \label{alg:addstc:connectuvstarts}
                \State \textbf{remove} $U$ and $V$ from $C_R$, and \textbf{add} $U \cup V$ to $C_R$ \label{alg:addstc:mergecomponents}
                \For{\textbf{each} $u' \in U$ and \textbf{each} $v' \in V$}                         \label{alg:addstc:mergeloop}
                    \State \textbf{add} $R(u', v')$ and $R(v', u')$ to $J$                          \label{alg:addstc:merge}
                \EndFor
            \EndIf                                                                                  \label{alg:addstc:connectuvends}
        \EndFor                                                                                     \label{alg:addstc:loopends}
        \State \Return $J$
    \EndFunction
\end{algorithmiccont}
\end{algorithm}
\vspace{-0.9cm}
\begin{algorithm}[H]
\caption{$\del{\mathsf{stc}(R)}{\ipos}{\ineg}{\Delta}{\Cnr}$}\label{alg:delstc}
\begin{algorithmiccont}
    \State $J \defeq \emptyset$
    \For{\textbf{each} $U \in C_R$ where $\exists R(u,v) \in \Delta$ s.t.\ $\{ u,v \} \subseteq U$} \label{alg:delstc:loopcomponents}
        \For{\textbf{each} $u' \in U$ and \textbf{each} $v' \in U$}                                 \label{alg:delstc:loopnodes}
            \If{$\Cnr(R(u',v')) = 0$}                                                               \label{alg:delstc:check}
                \textbf{add} $R(u',v')$ to $J$                                                      \label{alg:delstc:incj}
            \Else\;
                \textbf{add} $R(u',v')$ to $Y_R$                                                    \label{alg:delstc:incy}
            \EndIf
        \EndFor
        \State \textbf{remove} $U$ from $C_R$                                                       \label{alg:delstc:deltam}
    \EndFor
    \State \Return $J \setminus \Delta$
\end{algorithmiccont}
\end{algorithm}
\vspace{-0.9cm}
\begin{algorithm}[H]
\caption{$\red{\mathsf{stc}(R)}{\ipos}{\ineg}{\Delta}$}\label{alg:redstc}
\begin{algorithmiccont}
    \State $J \defeq \textsc{CloseEdges}(Y_R) \cap \Delta$
    \State $Y_R \defeq \emptyset$
    \State \Return $J$
\end{algorithmiccont}
\end{algorithm}
\end{table}

\section{Evaluation}\label{sec:evaluation}

\begin{table*}[t]
\centering
\newcommand{\mc}[1]{\multicolumn{1}{c|}{#1}}
\newcommand{\mce}[1]{\multicolumn{1}{c}{#1}}
\begin{tabular}{c|r|r|r|r|r|r|r|r|r}
    \hline
    Benchmark       & \mc{$|E|$}    & \mc{$|I|$}    & \mc{$S$}  & \mc{$|\nrstrat{\Pi}{}|$}  & \mc{$|\rstrat{\Pi}{}|$}   & \mc{$|$TC$|$} & \mc{$|$STC$|$}    & \mc{Mat-Mod}  & \mce{Mat} \\
    \hline
    Claros-LE       & 18.8 M        & 533.3 M       & 11        & 1031                      & 306                       & 27            & 2                 & 733.55        & 3593.32   \\
    \hline
    LUBM-LE         & 133.6 M       & 332.6 M       & 5         & 85                        & 22                        & 1             & 2                 & 291.90        & 1100.22   \\
    \hline
    DBpedia-SKOS    & 5.0 M         & 97.0 M        & 5         & 26                        & 15                        & 2             & 1                 & 103.23        & 3623.37   \\
    \hline
    DAG-R           & 0.1 M         & 22.9 M        & 1         & 1                         & 1                         & 1             & 0                 & 29.60         & 3238.86   \\
    \hline
\end{tabular}
\caption{Running times for materialisation computation (seconds)}\label{t1}
\end{table*}

\begin{table*}[t]
\centering
\newcommand{\mc}[1]{\multicolumn{2}{c|}{#1}}
\newcommand{\mce}[1]{\multicolumn{2}{c}{#1}}
\newcommand{\sca}[1]{\multicolumn{1}{c}{#1}}
\newcommand{\scb}[1]{\multicolumn{1}{c|}{#1}}
\begin{tabular}{c|r|r|r|r|r|r}
    \hline
    \multirow{2}{*}{Benchmark}  & \mc{Small Deletions}                      & \mc{Small Insertions}                     & \mce{Large Deletions}                     \\
                                & \sca{DRed$^c$-Mod}    & \scb{DRed$^c$}    & \sca{DRed$^c$-Mod}    & \scb{DRed$^c$}    & \sca{DRed$^c$-Mod}    & \sca{DRed$^c$}    \\
    \hline
    Claros-LE                   & 0.93                  &  1035.28          & 0.17                  & 0.80              & 314.33                & 3616.93           \\
    \hline
    LUBM-LE                     & 0.32                  &  3.87             & 0.01                  & 0.01              & 182.93                & 1369.77           \\
    \hline
    DBpedia-SKOS                & 21.77                 & 691.32            & 0.20                  & 2.78              & 111.28                & 3826.87           \\
    \hline
    DAG-R                       & 64.92                 & 3005.11           & 14.56                 & 116.78            & 62.48                 & 4316.71           \\
    \hline
\end{tabular}
\caption{Running times for incremental maintenance (seconds)}\label{t2}
\end{table*}

We have implemented our modular materialisation and incremental maintenance
algorithms, as well as the semina\"ive materialisation and the DRed$^c$
algorithms, and we have compared their performance empirically.

\subsubsection{Test Benchmarks} We used the following real-world and synthetic
benchmarks in our tests. LUBM \cite{guo2005lubm} is a well-known benchmark that
models individuals and organisations in a university domain. Claros describes
archeological artefacts. We used the LUBM and Claros datasets with the
\emph{lower bound extended} (\mbox{-LE}) programs by
\citeauthor{motik2014parallel} (\citeyear{motik2014parallel}); roughly
speaking, we converted a subset of the accompanying OWL ontologies into datalog
and manually extended them with several `difficult rules'. DBpedia
\cite{lehmann2015dbpedia} contains structured information extracted from
Wikipedia. DBpedia represents Wikipedia categories using the SKOS vocabulary
\cite{miles2009skos}, which defines several transitive properties. We used the
datalog subset of the SKOS RDF schema. Moreover, the materialisation of
DBpedia-SKOS is too large to fit into the memory of our test server, so we used
a random sample of the DBpedia dataset consisting of five million facts.
Finally, DAG-R is a synthetic benchmark consisting of a randomly generated
dataset containing a directed acyclic graph with 10k nodes and 100k edges, and
a program that axiomatises the path relation as transitive. Table~\ref{t1}
shows the numbers of explicit facts ($|E|$), derived facts ($|I|$), strata
($S$), nonrecursive rules ($|\nrstrat{\Pi}{}|$), recursive rules
($|\rstrat{\Pi}{}|$), transitivity modules ($|$TC$|$), and
symmetric--transitive modules ($|$STC$|$) for each benchmark.

\subsubsection{Test Setup and Results} We conducted all experiments on a Dell
PowerEdge R730 server with 512 GB RAM and two Intel Xeon E5-2640 2.6~GHz
processors running Fedora 27, kernel version 4.17.6. For each benchmark, we
loaded the test data into our system and then compared the performance of our
modular algorithms with the semina\"ive and DRed$^c$ algorithms using the
following methodology.

We first computed the materialisation and measured the wall-clock time. The
results are shown in Table~\ref{t1}. We then conducted two groups of
incremental reasoning tests.

In the first group, we tested the performance of our incremental algorithms on
small changes. To this end, we used uniform sampling to select ten subsets
${E_i \subseteq E}$, ${1 \leq i \leq 10}$, each consisting of 1000 facts from
the input dataset. We deleted and then reinserted $E_i$ for each $i$ while
measuring the wall-clock times, and then we computed the average times for
deletion and insertion over the ten samples. The results are shown in the
`Small Deletions' and `Small Insertions' columns of Table~\ref{t2},
respectively.

In the second group, we tested the performance of incremental algorithms on
large deletions. To this end, we used uniform sampling to select a subset ${E^-
\subseteq E}$ containing 25\% of the explicit facts, and we measured the
wall-clock time needed to delete $E^-$ from the materialisation. The results
are shown in the `Large Deletions' column of Table~\ref{t2}. We did not
consider large insertions because our algorithms handle insertion in the same
way as materialisation, so the relative performance of our algorithms should be
similar to the performance of materialisation shown in Table~\ref{t1}.

\subsubsection{Discussion} Mat-Mod significantly outperformed Mat on all test
inputs. For example, Mat-Mod was several times faster than Mat on Claros-LE and
LUBM-LE. The programs of both benchmarks contain transitivity and
symmetric--transitivity modules, which are efficiently handled by our custom
algorithm. The performance improvement is even more significant for
DBpedia-SKOS and DAG-R: Mat-Mod is more than 30 times faster than Mat on
DBpedia-SKOS, and the difference reaches two orders of magnitude on DAG-R. In
fact, DBpedia contains long chains/cycles over the $\emph{skos:broader}$
relation~\cite{bishop2011factforge}, which is axiomatised as transitive in
SKOS. Mat-Mod outperforms Mat in this case since our custom algorithm for
transitivity skips a large number of rule instances. The same observation
explains the superior performance of DAG-R.

Similarly, DRed$^c$-Mod considerably outperformed DRed$^c$ on small deletions:
the performance speedup ranges from around ten times on LUBM-LE to three orders
of magnitude on Claros-LE. The program of Claros-LE contains a
symmetric--transitive closure module for the predicate \emph{relatedPlaces},
and the materialisation contains large cliques of constants connected to each
other via this predicate. Thus, when a \emph{relatedPlaces(a,b)} fact is
deleted, DRed$^c$ can end up considering up to $n^3$ rule instances where $n$
is the number of constants in the clique containing $a$ and $b$. In contrast,
our custom algorithm for this module maintains a connected component for the
clique and requires only up to $n^2$ steps. It is worth noticing that, while
DRed$^c$-Mod significantly outperforms DRed$^c$ on DAG-R, the incremental
update times for small deletion were larger than both the update times for
large deletions and even for the initial materialisation. This is because
deleting one thousand edges from the graph (`Small Deletion') caused a large
part of the materialisation to be overdeleted and rederived again. In contrast,
when 25\% of the explicit facts are deleted (`Large Deletion'), a larger
propertion of the materialisation is overdeleted, but only a few facts are
rederived. For DRed$^c$ the situation is similar, but rederivation in DRed$^c$
benefits from a global recursive counter (at the expense of considering each
applicable rule instance), which makes small deletion still faster than large
deletion and initial materialisation. Finally, as shown in Table \ref{t2},
DRed$^c$-Mod scaled well and maintained its advantage over DRed$^c$ on large
deletions.

Incremental insertions are in general easier to handle than deletions since
during insertion the algorithms can rely on the whole materialisation to prune
the propagation of facts whereas during deletion the algorithms can only rely
on the nonrecursive counters of facts to do the same. This is clearly reflected
in Table \ref{t2}. Nevertheless, in our tests for small insertions,
DRed$^c$-Mod was several times faster than DRed$^c$ in all cases but LUBM-LE,
for which both algorithms updated the materialisation instantateously.

\section{Conclusion}

We have proposed a modular framework for the computation and maintenance of
datalog materialisations. The framework allows integrating custom algorithms
for specific types of rules with standard datalog reasoning methods. Moreover,
we have presented such custom algorithms for programs axiomatising the
transitive and the symmetric--transitive closure of a relation. Finally, we
have shown empirically that our algorithms typically significantly outperform
then existing ones, sometimes by orders of magnitude. In future, we plan to
extend our framework also to the B/F$^c$ algorithm, which eliminates
overdeletion by eagerly checking alternative derivations. This could
potentially be useful in cases such as DBpedia-SKOS and DAG-R, where
overdeletion is a major source of inefficiency.

\section*{Acknowledgements}\label{sec:acknowledgements}

We thank David Tena Cucala for his help with obtaining the SKOS benchmark. This
work was supported by the EPSRC projects AnaLOG, DBOnto, and ED$^3$.

\bibliographystyle{aaai}
\bibliography{references}

\iftoggle{withappendix}{
\newpage
\onecolumn
\appendix
\section{Appendix}\label{sec:proof}

\subsection{Proof of Theorem~\ref{theorem:correctness-mat}}

\correctnessmat*

\begin{proof}
We first prove a property about each $\addfn{s,i}$ function, which will be used
later to establish the correctness of the algorithm. More specifically, for
datasets $\ipos$, $\ineg$, and $\Delta$ with ${\Delta \subseteq \ipos}$, let
${\Delta_0 = \Delta}$ and ${J_0 = \emptyset}$; for ${k > 0}$, let
\begin{displaymath}
    \Delta_{k+1} = \apply{\rstrat{\Pi}{s,i}}{\ipos \cup J_k, \ineg \appargs \Delta_k} \setminus (\ipos \cup J_k) \quad \text{and} \quad  J_{k+1} = J_k \cup \Delta_{k+1};
\end{displaymath}
and let ${\semi{\rstrat{\Pi}{s,i}}{\ipos}{\ineg}{\Delta} = J_k}$ for $k$ such
that ${J_k = J_{k+1}}$. We show that
$\semi{\rstrat{\Pi}{s,i}}{\ipos}{\ineg}{\Delta}$ is the smallest set of facts
$J$ such that $\apply{\rstrat{\Pi}{s,i}}{\ipos \cup J, \ineg \appargs \Delta
\cup J} \subseteq \ipos \cup J$ holds---that is,
property~\eqref{semicomputesadd} holds.
\begin{align}
    \add{s,i}{\ipos}{\ineg}{\Delta} = \semi{\rstrat{\Pi}{s,i}}{\ipos}{\ineg}{\Delta}    \label{semicomputesadd}
\end{align}

To simplify the notation, let ${J =
\semi{\rstrat{\Pi}{s,i}}{\ipos}{\ineg}{\Delta}}$; we first prove that
$\apply{\rstrat{\Pi}{s,i}}{\ipos \cup J, \ineg \appargs \Delta \cup J}
\subseteq \ipos \cup J$ holds. To this end, consider an arbitrary fact ${F \in
\apply{\rstrat{\Pi}{s,i}}{\ipos \cup J, \ineg \appargs \Delta \cup J}}$. By the
definition of the latter, there exist a rule ${r \in \rstrat{\Pi}{s,i}}$ and
substitution $\sigma$ such that ${\pbody{r}\sigma \subseteq \ipos \cup J}$,
${\nbody{r}\sigma \cap \ineg = \emptyset}$, ${\pbody{r}\sigma \cap (\Delta \cup
J) \neq \emptyset}$, and ${F = \head{r}\sigma}$ all hold. Let $k$ be the
smallest index such that ${\pbody{r}\sigma \subseteq \ipos \cup J_k}$ and
${\pbody{r}\sigma \cap (\Delta \cup J_k) \neq \emptyset}$ hold. Then,
${\pbody{r}\sigma \cap \Delta_k \neq \emptyset}$ must hold, or $k$ is not the
smallest such index. This implies ${F \in \apply{\rstrat{\Pi}{s,i}}{\ipos \cup
J_k, \ineg \appargs \Delta_k} \subseteq \ipos \cup J}$, as required.

Now we show that, for each dataset $J'$ where ${\apply{\rstrat{\Pi}{s,i}}{\ipos
\cup J', \ineg \appargs \Delta \cup J'} \subseteq \ipos \cup J'}$ holds, we
have ${J \subseteq J'}$---that is, $J$ is the smallest such dataset. To this
end, we show by induction on $k$ that ${\Delta_k \subseteq J'}$ holds for
$k>0$. For the induction base, ${\Delta_0 = \Delta}$ and ${J_0 = \emptyset}$
imply ${\Delta_1 = \apply{\rstrat{\Pi}{s,i}}{\ipos, \ineg \appargs \Delta}
\setminus \ipos \subseteq \apply{\rstrat{\Pi}{s,i}}{\ipos \cup J', \ineg
\appargs \Delta \cup J'} \setminus \ipos \subseteq J'}$. For the inductive
step, consider arbitrary $k>1$ where ${\Delta_{k'} \subseteq J'}$ holds for
each $k'$ with ${1 \leq k' < k}$. Then, we clearly have ${J_{k-1} = \bigcup_{1
\leq k' < k}{\Delta_{k'}} \subseteq J'}$. But then, ${\Delta_k =
\apply{\rstrat{\Pi}{s,i}}{\ipos \cup J_{k-1}, \ineg \appargs \Delta_{k-1}}
\setminus (\ipos \cup J_{k-1}) \subseteq \apply{\rstrat{\Pi}{s,i}}{\ipos \cup
J', \ineg \appargs J'} \subseteq \ipos \cup J'}$ holds, as required.

\medskip

We now proceed with the proof of our main claim. Let $I^0 = \emptyset$.
Moreover, for each $1 \leq s \leq S$ where $S$ is the largest stratum index,
let $I^s_0, I^s_1, \dots$ be the sequence of sets where $I^s_0 = I^{s-1} \cup
(E \cap \Out{s})$, and $I^s_i = I^s_{i-1} \cup
\apply{\strat{\Pi}{s}}{I^s_{i-1}}$ for each $i > 0$. Index $k$ clearly exists
for which the fixpoint is reached (i.e., $I^s_{k} = I^s_{k+1}$ holds), so we
let $I^s = I^s_k$. Finally, let $I = I^S$. It is straightforward to see that $I
= \mat{\Pi}{E}$---that is, $I$ is the materialisation of $\Pi$ w.r.t.\ $E$.

Consider a run of Algorithm~\ref{and:matmod} on $\Pi$, $\lambda$, and $E$. Let
${I^0|_{\mathsf{mod}} = \emptyset}$, and for each ${1 \leq s \leq S}$, let
${I^s|_{\mathsf{mod}}}$ be the value of $I$ after the loop of
lines~\ref{and:matmod:looponsstarts}--\ref{and:matmod:looponsends} finishes for
stratum $s$. We show by induction on $s$ that property \eqref{IsequalsI} holds
for ${0 \leq s \leq S}$. Then, property~\eqref{and:matmod:looponsstarts} for
${s = S}$ and ${I^S = I = \mat{\Pi}{E}}$ jointly imply the correctness of the
algorithm.
\begin{align}
    I^s|_{\mathsf{mod}} = I^s \label{IsequalsI}
\end{align}

The base case where $s = 0$ is trivial since both sets are empty. For the
inductive step, consider an arbitrary $s$ with $1 \leq s \leq S$ such that
\eqref{IsequalsI} holds for $s-1$, and we show that \eqref{IsequalsI} holds for
$s$ as well. To this end, consider the execution of
lines~\ref{and:matmod:init}--\ref{and:matmod:looponsends} for stratum $s$. For
each $j > 0$, let $\Delta_i|_j$ and $\Delta|_j$ be the values of $\Delta_i$
(for $1 \leq i \leq n(s)$) and $\Delta$ when the $j$th iteration of
lines~\ref{and:matmod:innerloopstarts}--\ref{and:matmod:innerloopends} starts.
We show that property~\eqref{inductionfors} holds. Then, the way $I$ is updated
in line~\ref{and:matmod:innerloopstarts} ensures that
property~\eqref{IsequalsI} holds.
\begin{align}
    I^{s-1}|_{\mathsf{mod}} \cup \bigcup_{j > 0}{\Delta|_j} = I^s \label{inductionfors}
\end{align}

For the $\subseteq$ direction of \eqref{inductionfors},
$I^{s-1}|_{\mathsf{mod}} = I^{s-1} \subseteq I^s$ holds by the induction
assumption for $\eqref{IsequalsI}$. Next we prove $\bigcup_{j > 0}{\Delta|_j}
\subseteq I^s$ by induction on $j$.
\begin{itemize}
    \item For the base case where $j = 1$, line~\ref{and:matmod:updatedelta}
    ensures that $\Delta_1 = (E \cap \Out{s}) \cup
    \apply{\nrstrat{\Pi}{s}}{I^{s-1}|_{\mathsf{mod}}}$. But then, the induction
    assumption $I^{s-1}|_{\mathsf{mod}} = I^{s-1}$ and the definition of $I^s$
    jointly imply $\Delta_1 \subseteq I^s$.

    \item For the inductive step, consider arbitrary $j > 1$ such that
    $\Delta_k \subseteq I^s$ holds for each $1 \leq k < j$. Then,
    line~\ref{and:matmod:updatedeltai} and the induction assumption for
    \eqref{IsequalsI} ensure that $\Delta_i|_j = \add{s,i}{I^{s-1} \cup
    \bigcup_{1 \leq k < j}{\Delta|_j}}{I^{s-1} \cup \bigcup_{1 \leq k <
    j}{\Delta|_k}}{\Delta|_{j-1} \setminus \Delta_i|_{j-1}}$. By
    property~\eqref{semicomputesadd} we have $\Delta_i|_j =
    \semi{\rstrat{\Pi}{s,i}}{I^{s-1} \cup \bigcup_{1 \leq k <
    j}{\Delta|_k}}{I^{s-1} \cup \bigcup_{1 \leq k <
    j}{\Delta|_j}}{\Delta|_{j-1} \setminus \Delta_i|_{j-1}}$. Then, the
    induction assumption and the definition of $I^s$ imply $I^{s-1} \cup
    \bigcup_{1 \leq k < j}{\Delta|_k} \subseteq I^s$. Now let the sequences of
    $\Delta_m$ and $J_m$ with $m \geq 0$ be defined in the same way as in the
    definition for the $\semifn$ function. We prove by induction on $m$ that
    $\Delta_m \subseteq I^s$ and $J_m \subseteq I^s$ hold, then the definition
    of $\semifn$ implies $\Delta_i|_j \subseteq I^s$.
    \begin{itemize}
        \item We have $\Delta_0 = \Delta|_{j-1} \setminus \Delta_i|_{j-1}
        \subseteq \Delta|_{j-1} \subseteq I^s$ and $J_0 = \emptyset \subseteq
        I^s$, so the induction base where $m=0$ clearly holds.

        \item For the inductive step, consider arbitrary $m > 0$ such that
        $\Delta_{m-1} \subseteq I^s$ and $J_{m-1} \subseteq I^s$ hold. But
        then, by definition we have $\Delta_m \subseteq
        \apply{\rstrat{\Pi}{s,i}}{I^{s-1} \cup \bigcup_{1 \leq k <
        j}{\Delta|_k} \cup J_{m-1}, I^{s-1} \cup \bigcup_{1 \leq k <
        j}{\Delta|_k}} \subseteq \apply{\rstrat{\Pi}{s,i}}{I^s, I^{s-1} \cup
        \bigcup_{1 \leq k < j}{\Delta|_k}}$. Facts in $I^s \setminus I^{s-1}$
        all belong to stratum $s$, so they will not affect the evaluation of
        negative body atoms from rules in stratum $s$. Therefore we have
        $\Delta_m \subseteq \apply{\rstrat{\Pi}{s,i}}{I^s, I^s} =
        \apply{\rstrat{\Pi}{s,i}}{I^s} \subseteq I^s$, as required.
        Furthermore, by definition $J_m = J_{m-1} \cup \Delta_m$, together with
        the induction assumption $J_{m-1} \subseteq I^s$ this ensures that $J_m
        \subseteq I^s$ holds as well.
    \end{itemize}
    Now line~\ref{and:matmod:combinedelta} and the fact that $D_i|_j \subseteq
    I^s$ holds for each $1 \leq i \leq n(s)$ jointly imply $\Delta|_j =
    \bigcup_{1 \leq i \leq n(s)}{\Delta_i|_j} \subseteq I^s$, as required. This
    completes our proof for $\bigcup_{j > 0}{\Delta|_j} \subseteq I^s$.
\end{itemize}

For the $\supseteq$ direction of property~\eqref{inductionfors},
we prove by induction on $i$ that $I^s_i \subseteq I^{s-1}|_{\mathsf{mod}} \cup
\bigcup_{j > 0}{\Delta|_j}$ holds for $i \geq 0$.
\begin{itemize}
    \item For the base case, we have $I^s_0 = I^{s-1} \cup (E \cap \Out{s})$.
    But then, line~\ref{and:matmod:updatedelta} ensures $E \cap \Out{s}
    \subseteq \Delta|_1$, which together with the induction assumption for
    \eqref{IsequalsI} implies $I^s_0 \subseteq I^{s-1}|_{\mathsf{mod}} \cup
    \bigcup_{j > 0}{\Delta|_j}$.

    \item For the induction step, consider arbitrary $i > 0$ such that
    $I^s_{i-1} \subseteq I^{s-1}|_{\mathsf{mod}} \cup \bigcup_{j >
    0}{\Delta|_j}$ holds, and we would like to show that $I^s_{i} \subseteq
    I^{s-1}|_{\mathsf{mod}} \cup \bigcup_{j > 0}{\Delta|_j}$ holds as well. By
    the induction assumption for $i-1$ and the fact that $I^s_i = I^s_{i-1}
    \cup \apply{\strat{\Pi}{s}}{I^s_{i-1}}$, it is enough to prove
    $\apply{\strat{\Pi}{s}}{I^s_{i-1}} \subseteq I^{s-1}|_{\mathsf{mod}} \cup
    \bigcup_{j > 0}{\Delta|_j}$. To this end, consider arbitrary $F \in
    \apply{\strat{\Pi}{s}}{I^s_{i-1}}$. There are two cases here. If $F \in
    \apply{\nrstrat{\Pi}{s}}{I^s_{i-1}}$---that is, $F$ can be derived by a
    nonrecursive rule, then we have $F \in \apply{\nrstrat{\Pi}{s}}{I^{s-1}}$.
    But then, the induction assumption for \eqref{IsequalsI} and
    line~\ref{and:matmod:updatedelta} of the algorithm ensure $F \in
    \Delta|_1$. If $F \in \apply{\rstrat{\Pi}{s}}{I^s_{i-1}}$, then there
    exists a module with index $k$ such that $F \in
    \apply{\rstrat{\Pi}{s,k}}{I^s_{i-1}}$. By the definition of rule
    application, there exist rule $r$ and its instance $r'$ such that $r \in
    \rstrat{\Pi}{s,k}$, $\pbody{r'} \subseteq I^s_{i-1}$, and $\nbody{r'} \cap
    I^s_{i-1} = \emptyset$ all hold. Since $I^s_{i-1} \subseteq
    I^{s-1}|_{\mathsf{mod}} \cup \bigcup_{j > 0}{\Delta|_j}$ holds by the
    induction assumption, let $j'$ be the largest index $j$ such that
    $\pbody{r'} \cap \Delta|_j \neq \emptyset$. Then, if $\pbody{r'} \cap
    (\Delta|_{j'} \setminus \Delta_k|_{j'}) \neq \emptyset$, then
    Definition~\ref{def:add} ensures that $F$ is added to $\Delta_{k}|_{j'+1}$
    in line~\ref{and:matmod:updatedeltai} during the execution of the $j'$th
    iteration of
    lines~\ref{and:matmod:innerloopstarts}--\ref{and:matmod:innerloopends}; if
    $\pbody{r'} \cap (\Delta|_{j'} \setminus \Delta_k|_{j'}) = \emptyset$, then
    we have $\pbody{r'} \subseteq I^{s-1}|_{\mathsf{mod}} \cup \bigcup_{0 < j <
    j'}{\Delta|_j} \cup \Delta_k|_{j'}$, so Definition~\ref{def:add} ensures
    that $F$ is added to $\Delta_{k}|_{j'}$ in
    line~\ref{and:matmod:updatedeltai} during the execution of the $(j'-1)$th
    iteration of
    lines~\ref{and:matmod:innerloopstarts}--\ref{and:matmod:innerloopends}.
    Either way, we have $F \in I^{s-1}|_{\mathsf{mod}} \cup \bigcup_{j >
    0}{\Delta|_j}$. Since the choice of $F$ is arbitrary, we have
    $\apply{\rstrat{\Pi}{s}}{I^s_{i-1}} \subseteq I^{s-1}|_{\mathsf{mod}} \cup
    \bigcup_{j > 0}{\Delta|_j}$, as required.
\end{itemize}

This completes our proof for the correctness of the algorithm.

Next we show that if all $\addfn{s,i}$ use the semina{\"i}ve strategy---that
is, each $\add{s,i}{\ipos}{\ineg}{\Delta}$ is computed in the same way that
$\semi{\rstrat{\Pi}{s,i}}{\ipos}{\ineg}{\Delta}$ is constructed, then each
applicable rule instance is considered at most once. To this end, consider a
run of Algorithm~\ref{and:matmod}. First, please note that the program is
processed in a stratum-by-stratum manner, so no applicable rule instance will
be considered in two distinct iterations of
lines~\ref{and:matmod:looponsstarts}--\ref{and:matmod:looponsends}. Now
consider an arbitrary stratum index $s$ and the iteration of
lines~\ref{and:matmod:looponsstarts}--\ref{and:matmod:looponsends} for $s$, the
only places that consider rule instances are lines~\ref{and:matmod:updatedelta}
and \ref{and:matmod:updatedeltai}. Line~\ref{and:matmod:updatedelta} handles
nonrecursive rules whereas line~\ref{and:matmod:updatedeltai} handles recursive
rules, so no rule instance will be considered in both places.
Line~\ref{and:matmod:updatedelta} is only executed once for $s$ while
line~\ref{and:matmod:updatedeltai} can be executed multiple times. Thus it is
sufficient to show that for stratum $s$, line~\ref{and:matmod:updatedeltai}
never repeatedly consider an applicable rule instance. The way
$\semi{\rstrat{\Pi}{s,i}}{\ipos}{\ineg}{\Delta}$ is constructed ensures that
one application of $\addfn{s,i}$ does not repeat rule intances itself. Now
consider the $m$th and the $n$th iterations of
lines~\ref{and:matmod:innerloopstarts}--\ref{and:matmod:innerloopends} where we
have $m \neq n$. We would like to show that for each $i$, the application of
$\add{s,i}{I^{s-1}|_{\mathsf{mod}} \cup \bigcup_{0 < j \leq
m}{\Delta|_j}}{I^{s-1}|_{\mathsf{mod}} \cup \bigcup_{0 < j \leq
m}{\Delta|_j}}{\Delta|_m \setminus \Delta_i|_m}$ and the application of
$\add{s,i}{I^{s-1}|_{\mathsf{mod}} \cup \bigcup_{0 < j \leq
n}{\Delta|_j}}{I^{s-1}|_{\mathsf{mod}} \cup \bigcup_{0 < j \leq
n}{\Delta|_j}}{\Delta|_n \setminus \Delta_i|_n}$ do not repeat rule instances.
Without loss of generality assume that $m < n$ holds. By the construction of
$\semi{\rstrat{\Pi}{s,i}}{\ipos}{\ineg}{\Delta}$ we know that for the former,
each applicable rule instance must have at least one positive body atom in
$(\Delta|_m \setminus \Delta_i|_m) \cup \Delta_i|_{m+1}$, and for the latter,
each applicable rule instance must have one positive body atom in $(\Delta|_n
\setminus \Delta_i|_n) \cup \Delta_i|_{n+1}$. It is straightforward to see that
these two sets are disjoint, so no applicable rule instance will be considered
by the these applications of $\addfn{}$, and this completes our proof for the
second half of Theorem~\ref{theorem:correctness-mat}.
\end{proof}

\subsection{Proof of Theorem~\ref{theorem:correctness-dredmod}}

\correctnessdredcmod*

\begin{proof}
For a program $\Pi$ and datasets $\ipos$, $\ineg$, and $\Delta$ with $\Delta
\subseteq \ipos$, let $\Delta_0 = \Delta$ and $J_0 = \emptyset$; moreover, for
$i > 0$, let $\Delta_{i+1} = \apply{\Pi}{\ipos \setminus J_i, \ineg \appargs
\Delta_i} \setminus (\Delta \cup J_i)$ and $J_{i+1} = J_i \cup \Delta_i$; let
$\invsemi{\Pi}{\ipos}{\ineg}{\Delta} = J_i \setminus \Delta$ for $i$ such that
$J_i = J_{i+1}$ holds. We show that $J = \invsemi{\Pi}{\ipos}{\ineg}{\Delta}$
is the smallest dataset satisfying $\apply{\Pi}{\ipos, \ineg \appargs \Delta
\cup J} \subseteq \Delta \cup J$.

We show by induction on $i \geq 0$ that $\apply{\Pi}{\ipos, \ineg \appargs
\Delta \cup J_i} \subseteq \Delta \cup J$ holds. The base case where $i=0$
trivially holds. For the inductive step, consider arbitrary $i > 0$ such that
$\apply{\Pi}{\ipos, \ineg \appargs \Delta \cup J_{i-1}} \subseteq \Delta \cup
J$ holds. For each $F \in \apply{\Pi}{\ipos, \ineg \appargs \Delta \cup J_i}$,
there exist rule $r \in \Pi$ and substitution $\sigma$ such that
$\pbody{r}\sigma \in \ipos$, $\nbody{r}\sigma \cap \ineg = \emptyset$,
$\pbody{r}\sigma \cap (\Delta \cup J_i) \neq \emptyset$, and $F =
\head{r}\sigma$ all hold. If $\pbody{r}\sigma \cap (\Delta \cup J_{i-1}) \neq
\emptyset$, then the induction assumption ensures $F \in \Delta \cup J$.
Otherwise $\pbody{r}\sigma \cap \Delta_{i-1} \neq \emptyset$ holds. Now there
are two possibilities: if $\pbody{r}\sigma \subseteq \ipos \setminus J_{i-1}$,
then $F$ is derived in the contruction of $\Delta_i$; if $\pbody{r}\sigma \cap
J_{i-1} \neq \emptyset$, then the induction assumption ensures $F \in \Delta
\cup J$. Therefore, $\apply{\Pi}{\ipos, \ineg \appargs \Delta \cup J} \subseteq
\Delta \cup J$ holds, as required.

To see that $J$ is the smallest such set, let $J'$ be an arbitrary set
satisfying $\apply{\Pi}{\ipos, \ineg \appargs \Delta \cup J'} \subseteq \Delta
\cup J'$ and we prove by induction on $i$ that $\bigcup_{0 \leq j \leq
i}{\Delta_j} \setminus \Delta \in J'$ holds. The base case where $i=0$ clearly
holds since $\Delta_0 \setminus \Delta = \emptyset$. For the inductive step,
consider arbitrary index $i>0$ such that $\bigcup_{0 \leq j \leq
i-1}{\Delta_{j}} \setminus \Delta \subseteq J'$ holds. This implies
$\Delta_{i-1} \subseteq \Delta \cup J'$, which together with the definition of
$\Delta_i$ ensures $\Delta_i \subseteq \Delta \cup J'$, so the inductive step
holds. Therefore, $J = \invsemi{\Pi}{\ipos}{\ineg}{\Delta}$ is the smallest
dataset satisfying $\apply{\Pi}{\ipos, \ineg \appargs \Delta \cup J} \subseteq
\Delta \cup J$.

\medskip

For a program $\Pi$, datasets $\ipos$, $\ineg$, and $\Delta$ with ${\Delta
\subseteq \ipos}$, and a mapping $\Cnr$ of facts to nonnegative integers, let
${\Delta_0 = \Delta}$ and ${J_0 = \emptyset}$; moreover, for $i>0$, let
\begin{displaymath}
    \Delta_{i+1} = \{ F \in \apply{\Pi}{\ipos \setminus J_i, \ineg \appargs \Delta_i} \setminus (\Delta \cup J_i) | \Cnr(F) = 0 \} \quad \text{and} \quad J_{i+1} = J_i \cup \Delta_i;
\end{displaymath}
and let $\invsemic{\Pi}{\ipos}{\ineg}{\Delta}{\Cnr} = J_i \setminus \Delta$ for
$i$ such that $J_i = J_{i+1}$. We show that $J =
\invsemic{\Pi}{\ipos}{\ineg}{\Delta}{\Cnr}$ is the smallest dataset that
satisfies the following: for each $F \in \apply{\Pi}{\ipos, \ineg \appargs
\Delta \cup J}$, either $F \in \Delta \cup J$ or $\Cnr(F) > 0$ holds.

We prove induction on $i \geq 0$ that for each $F \in \apply{\Pi}{\ipos, \ineg
\appargs \Delta \cup J_i}$, either $F \in \Delta \cup J$ or $\Cnr(F) > 0$
holds. The base case where $i = 0$ clearly holds. For the inductive step,
consider arbitrary $i$ such that for each $F \in \apply{\Pi}{\ipos, \ineg
\appargs \Delta \cup J_{i-1}}$, either $F \in \Delta \cup J$ or $\Cnr(F) > 0$
holds. Then, for each fact $G \in \apply{\Pi}{\ipos, \ineg \appargs \Delta \cup
J_i}$, there exist a rule $r \in \Pi$ and a substitution $\sigma$ such that
$\pbody{r}\sigma \subseteq \ipos$, $\nbody{r}\sigma \cap \ineg = \emptyset$,
$\pbody{r}\sigma \cap (\Delta \cup J_i) \neq \emptyset$, and $G =
\head{r}\sigma$ all hold. Now if $\pbody{r}\sigma \cap (\Delta \cup J_{i-1})$,
then the induction assumption ensures that either $G \in \Delta \cup J$ or
$\Cnr(F) > 0$ holds. Otherwise we have $\pbody{r}\sigma \cap (J_i \setminus
J_{i-1}) = \pbody{r}\sigma \cap \Delta_{i-1} \neq \emptyset$. There are two
possibilities: if $\pbody{r}\sigma \subseteq \ipos \setminus J_{i-1}$, then $G$
is derived in the construction of $\Delta_i$ and it is either guaranteed to be
in $\Delta \cup J$ or we have $\Cnr(G) > 0$; if $\pbody{r}\sigma \cap J_{i-1}
\neq \emptyset$, then the induction assumption ensures that either $G \in
\Delta \cup J$ or $\Cnr(G) > 0$ holds.

To see that $J$ is the smallest such set, let $J'$ be an arbitrary set
satisfying the following: for each $F \in \apply{\Pi}{\ipos, \ineg \appargs
\Delta \cup J'}$, either $F \in \Delta \cup J'$ or $\Cnr(F) > 0$ holds. We
prove by induction on $i$ that $\Delta_i \setminus \Delta \subseteq J' $ holds.
The base case where $i=0$ clearly holds since $\Delta_0 \setminus \Delta =
\emptyset$. For the inductive step, consider arbitrary index $i>0$ such that
$\Delta_{i-1} \setminus \Delta \subseteq J'$ holds. This implies $\Delta_{i-1}
\subseteq \Delta \cup J'$, which together with the definition of $\Delta_i$ and
the induction assumption ensures $\Delta_i \subseteq \Delta \cup J'$, so the
inductive step holds. Therefore, $J =
\invsemic{\Pi}{\ipos}{\ineg}{\Delta}{\Cnr}$ is indeed the smallest dataset
satisfying the above property.

\medskip

To see that our main claim holds, note that
$\invsemi{\Pi}{\ipos}{\ineg}{\Delta}$ captures overdeletion in DRed and
corresponds to the upper bound $J_u$ in Definition~\ref{def:del}, whereas
$\invsemic{\Pi}{\ipos}{\ineg}{\Delta}{\Cnr}$ captures overdeletion in DRed$^c$
and corresponds to the lower bound $J_l$ in Definition~\ref{def:del}. Then, the
correctness of Algorithm~\ref{alg:dredmod} follows from the correctness of DRed
and DRed$^c$.
\end{proof}

\subsection{Proof of Theorem~\ref{theorem:correctness-trans}}

\correctnesstrans*

We consider each of these functions in combination with each of the relevant
algorithms in a separate claim. We first consider Algorithm~\ref{and:matmod}.
For $I$ a dataset and $R$ a predicate, let $R[I]$ denote the set of all $R$
facts in $I$.

\begin{claim}\label{claimaddtc}
    If ${\mat{\Pi^{\mathsf{tc}(R)}}{X_R} = R[\ipos \setminus \Delta]}$ holds
    before each call to function ${\add{\mathsf{tc}(R)}{\ipos}{\ineg}{\Delta}}$
    in a run of Algorithm~\ref{and:matmod}, then
    ${\mat{\Pi^{\mathsf{tc}(R)}}{X_R \cup \Delta} = R[\ipos] \cup J}$ holds,
    where ${J = \add{\mathsf{tc}(R)}{\ipos}{\ineg}{\Delta}}$. Moreover,
    ${\semi{\Pi^{\mathsf{tc}(R)}}{\ipos}{\ineg}{\Delta} = J}$ holds.
\end{claim}

\begin{proof}
First we show that ${\mat{\Pi^{\mathsf{tc}(R)}}{X_R \cup \Delta} = R[\ipos]
\cup J}$ holds. For the $\subseteq$ direction of the equation, consider
arbitrary ${R(u,v) \in \mat{\Pi^{\mathsf{tc}(R)}}{X_R \cup \Delta}}$. If
${R(u,v) \in \mat{\Pi^{\mathsf{tc}(R)}}{X_R}}$, then ${R(u,v) \in R[\ipos
\setminus \Delta] \subseteq R[\ipos] \cup J}$ clearly holds. Otherwise, there
exists a shortest chain of $R$ facts in ${X_R \cup \Delta}$ that connects $u$
and $v$. We show by induction on the length of this chain that ${R(u,v) \in
R[\ipos] \cup J}$ holds and $R(u,v)$ is added to $Q$ at some point during the
execution of the algorithm. For the base case where the length is one, since
${R(u,v) \not \in \mat{\Pi^{\mathsf{tc}(R)}}{X_R}}$, we have ${R(u,v) \not \in
X_R}$. Thus, ${R(u,v) \in \Delta \in \ipos}$ holds, and $R(u,v)$ is added to
$Q$ in line~\ref{alg:addtc:init}. For the inductive step, consider a length
$i+1$ chain ${R(a_0,a_1), \dots, R(a_i, a_{i+1})}$. We consider two cases. In
the first case, we have $R(a_0,a_1) \in \Delta$ and $R(a_1,a_{i+1}) \in
\mat{\Pi^{\mathsf{tc}(R)}}{X_R} = R[\ipos \setminus \Delta]$. But then, $R(a_0,
a_{i+1})$ is added to $Q$ and $J$ in line~\ref{alg:addtc:addtoq}. In the second
case, we have $R(a_0, a_1) \in X_R \cup \Delta$, $R(a_1,a_{i+1}) \in
\mat{\Pi^{\mathsf{tc}(R)}}{X_R \cup \Delta}$, and $R(a_1,a_{i+1}) \not \in
\mat{\Pi^{\mathsf{tc}(R)}}{X_R}$. But then, there exists a chain of length $i$
that derives $R(a_1,a_{i+1})$, which by the induction assumption ensures that
$R(a_1,a_{i+1})$ is added to $Q$ at some point during the execution of the
algorithm. Lines~\ref{alg:addtc:prune} and \ref{alg:addtc:addtoqandj} then
ensure that $R(a_0,a_{i+1})$ is added to $Q$ and that $R(a_0,a_{i+1}) \in
R[\ipos] \cup J$ holds. We next prove ${J =
\semi{\Pi^{\mathsf{tc}(R)}}{\ipos}{\ineg}{\Delta}}$. Then
${\mat{\Pi^{\mathsf{tc}(R)}}{X_R} = R[\ipos \setminus \Delta]}$ implies
${\mat{\Pi^{\mathsf{tc}(R)}}{\ipos \setminus \Delta} = R[\ipos \setminus
\Delta]}$; similarly, ${\mat{\Pi^{\mathsf{tc}(R)}}{X_R \cup \Delta} = R[\ipos]
\cup J}$ implies ${\mat{\Pi^{\mathsf{tc}(R)}}{\ipos} =
\mat{\Pi^{\mathsf{tc}(R)}}{\ipos \cup J} = R[\ipos] \cup J}$. Moreover, it can
be easily verified by induction on the construction of
${\semi{\Pi^{\mathsf{tc}(R)}}{\ipos}{\ineg}{\Delta}}$ that
${\mat{\Pi^{\mathsf{tc}(R)}}{\ipos \setminus \Delta} \cup R[\Delta] \cup
\semi{\Pi^{\mathsf{tc}(R)}}{\ipos}{\ineg}{\Delta} =
\mat{\Pi^{\mathsf{tc}(R)}}{\ipos}}$---that is, the semina\"ive computation
correctly closes $\ipos$ with respect to $\Pi^{\mathsf{tc}(R)}$. Thus, ${J =
\semi{\Pi^{\mathsf{tc}(R)}}{\ipos}{\ineg}{\Delta}}$ holds, as required.
\end{proof}

That $\addfn{\mathsf{tc}(R)}$ captures $\Pi^{\mathsf{tc}(R)}$ during the
execution of Algorithm~\ref{and:matmod} follows from Claim~\ref{claimaddtc} and
property~\eqref{semicomputesadd}. We next show that, during the execution of
Algorithm~\ref{alg:dredmod}, functions $\addfn{\mathsf{tc}(R)}$,
$\delfn{\mathsf{tc}(R)}$, $\difffn{\mathsf{tc}(R)}$, and
$\redfn{\mathsf{tc}(R)}$ capture $\Pi^{\mathsf{tc}(R)}$. Note that the order of
the function calls is important for correctness, so we examine
$\difffn{\mathsf{tc}(R)}$ first. For each call to function
$\diff{\mathsf{tc}(R)}{\ipos}{\Deltapos}{\Deltaneg}$ in
line~\ref{alg:dredmod:reddel:diff}, we have $\Deltapos = D \setminus A$, which
contains only facts from previous strata. Thus,
$\apply{\Pi^{\mathsf{tc}(R)}}{\ipos \appargs \Deltapos, \Deltaneg} = \emptyset$
clearly holds since the only rule in $\Pi^{\mathsf{tc}(R)}$ is recursive. Our
implementation for $\diff{\mathsf{tc}(R)}{\ipos}{\Deltapos}{\Deltaneg}$ always
return empty set as well, so by Definition~\ref{def:diff}
$\difffn{\mathsf{tc}(R)}$ captures $\Pi^{\mathsf{tc}(R)}$ for the calls in
line~\ref{alg:dredmod:reddel:diff}. For the same reason,
$\difffn{\mathsf{tc}(R)}$ captures $\Pi^{\mathsf{tc}(R)}$ for the calls in
line~\ref{alg:dredmod:redinsert:diff} as well. We next focus on
$\delfn{\mathsf{tc}(R)}$, $\redfn{\mathsf{tc}(R)}$, and
$\addfn{\mathsf{tc}(R)}$.

\begin{claim}\label{claimdeltc}
    If ${\mat{\Pi^{\mathsf{tc}(R)}}{X_R} = \mat{\Pi^{\mathsf{tc}(R)}}{\ipos}}$
    holds before each call to
    $\del{\mathsf{tc}(R)}{\ipos}{\ineg}{\Delta}{\Cnr}$ in a run of
    Algorithm~\ref{alg:dredmod}, then ${\mat{\Pi^{\mathsf{tc}(R)}}{X_R
    \setminus \Delta} = \mat{\Pi^{\mathsf{tc}(R)}}{(\ipos \setminus \Delta)
    \setminus J}}$ holds, where ${J =
    \del{\mathsf{tc}(R)}{\ipos}{\ineg}{\Delta}{\Cnr}}$. Moreover, ${J_l
    \subseteq J \subseteq J_u}$ holds, where $J_l$ and $J_u$ are the lower and
    upper bounds from Definition~\ref{def:del}, respectively.
\end{claim}

\begin{proof}
The proof is analogous to the proof of Claim~\ref{claimaddtc} and is based on
the intuition that the function implements semina\"ive evaluation for the rule
$X(x,y) \wedge R(y,z) \rightarrow R(x,z)$.
\end{proof}

That $\delfn{\mathsf{tc}(R)}$ captures $\Pi^{\mathsf{tc}(R)}$ during the
execution of Algorithm~\ref{alg:dredmod} immediately follows from
Claim~\ref{claimdeltc} and Definition~\ref{def:del}.

\begin{claim}\label{claimredtc}
    If ${\mat{\Pi^{\mathsf{tc}(R)}}{X_R} = \mat{\Pi^{\mathsf{tc}(R)}}{\ipos
    \setminus \Delta}}$ holds before each call to
    $\red{\mathsf{tc}(R)}{\ipos}{\ineg}{\Delta}$ in a run of
    Algorithm~\ref{alg:dredmod}, then ${\mat{\Pi^{\mathsf{tc}(R)}}{X_R} =
    R[(\ipos \setminus \Delta)] \cup J}$ holds, where ${J =
    \red{\mathsf{tc}(R)}{\ipos}{\ineg}{\Delta}}$. Moreover, $J$ is the smallest
    dataset satisfying $\apply{\Pi^{\mathsf{tc}(R)}}{(\ipos \setminus \Delta)
    \cup J, \ineg} \cap \Delta \subseteq J$.
\end{claim}

\begin{proof}
First we show that $\mat{\Pi^{\mathsf{tc}(R)}}{X_R} = R[(\ipos \setminus
\Delta)] \cup J$ holds. The $\supseteq$ direction of the property is trivial:
each fact $R(u,v)$ added to $J$ in line~\ref{alg:deltc:addtoj} corresponds to a
chain of facts in $X_R$, so $R(u,v) \in \mat{\Pi^{\mathsf{tc}(R)}}{X_R}$ holds.
For the $\subseteq$ direction, consider arbitrary fact $R(u,v)$ such that
$R(u,v) \in \mat{\Pi^{\mathsf{tc}(R)}}{X_R}$ and $R(u,v) \not \in R[(\ipos
\setminus \Delta)]$ hold. Since the original materialisation $I$ is passed as
$\ipos$ when the function gets called, we have $R(u,v) \in \ipos \setminus
(\ipos \setminus \Delta) = \Delta$. But then,
lines~\ref{alg:deltc:findsource}--\ref{alg:deltc:addtoj} ensure that a chain of
facts in $X_R$ deriving $R(u,v)$ will be found and that $R(u,v)$ is added to
$J$. We now show that $J$ is the smallest dataset satisfying
$\apply{\Pi^{\mathsf{tc}(R)}}{(\ipos \setminus \Delta) \cup J, \ineg} \cap
\Delta \subseteq J$. Consider arbitrary fact $R(u,v) \in
\apply{\Pi^{\mathsf{tc}(R)}}{(\ipos \setminus \Delta) \cup J, \ineg} \cap
\Delta$. By $\mat{\Pi^{\mathsf{tc}(R)}}{X_R} = R[(\ipos \setminus \Delta)] \cup
J$, we have $R(u,v) \in \Delta \cap \mat{\Pi^{\mathsf{tc}(R)}}{X_R}$. But then,
lines~\ref{alg:deltc:findsource}--\ref{alg:deltc:addtoj} ensure $R(u,v) \in J$.
Moreover, it is straightforward to show by induction that any $J'$ satisfying
$\apply{\Pi^{\mathsf{tc}(R)}}{(\ipos \setminus \Delta) \cup J', \ineg} \cap
\Delta \subseteq J'$ must at least contain all facts in
$\mat{\Pi^{\mathsf{tc}(R)}}{\ipos \setminus \Delta} \cap \Delta =
\mat{\Pi^{\mathsf{tc}(R)}}{X_R} \setminus (\ipos \setminus \Delta) = J$---in
other words, $J$ is indeed the smallest dataset satisfying
$\apply{\Pi^{\mathsf{tc}(R)}}{(\ipos \setminus \Delta) \cup J, \ineg} \cap
\Delta \subseteq J$.
\end{proof}

That $\redfn{\mathsf{tc}(R)}$ captures $\Pi^{\mathsf{tc}(R)}$ during the
execution of Algorithm~\ref{alg:dredmod} immediately follows from
Claim~\ref{claimredtc} and Definition~\ref{def:red}. Finally, the proof of
$\addfn{\mathsf{tc}(R)}$ capturing $\Pi^{\mathsf{tc}(R)}$ during the execution
of Algorithm~\ref{alg:dredmod} is analogous to the proof of
Claim~\ref{claimaddtc} so we omit the details for the sake of brevity.

\subsection{Proof of Theorem~\ref{theorem:correctness-symtrans}}

\correctnesssymtrans*

We consider each of these functions in combination with each of the relevant
algorithms in a separate claim. We first consider Algorithm~\ref{and:matmod}.
For $C_R$ a set of sets (representing a set of connected components of edges
consisting of relation $R$), let ${\mathsf{Close}(C_R) = \bigcup_{U \in
C_R}{\bigcup_{u,v \in U}{\{R(u,v)\}}}}$.

\begin{claim}\label{claimadd}
    If ${\mathsf{Close}(C_R) = R[\ipos \setminus \Delta]}$ holds before each
    call to $\add{\mathsf{stc}(R)}{\ipos}{\ineg}{\Delta}$ in a run of
    Algorithm~\ref{and:matmod}, then $C_R$ is updated so that
    ${\mathsf{Close}(C_R) = R[\ipos] \cup J}$ holds, where ${J =
    \add{\mathsf{stc}(R)}{\ipos}{\ineg}{\Delta}}$. Moreover, ${J =
    \semi{\Pi^{\mathsf{stc}(R)}}{\ipos}{\ineg}{\Delta}}$ holds.
\end{claim}

\begin{proof}
First we show that $\mathsf{Close}(C_R) = R[\ipos] \cup J$ holds after the
function call. For the $\subseteq$ direction, please note that $C_R$ can be
updated in only three places---lines~\ref{alg:addstc:createucomonent},
\ref{alg:addstc:createvcomonent}, and \ref{alg:addstc:mergecomponents}. In
line~\ref{alg:addstc:createucomonent} the first command adds a new component $U
= \{u\}$ to $C_R$. This will add $R(u,u)$ to $\mathsf{Close}(C_R)$. But then,
the second command in line~\ref{alg:addstc:createucomonent} ensures that
$R(u,u) \in \ipos \cup J$ holds. The same reasoning applies to
line~\ref{alg:addstc:createvcomonent}. In line~\ref{alg:addstc:mergecomponents}
two components $U$ and $V$ are merged; but then,
lines~\ref{alg:addstc:mergeloop}--\ref{alg:addstc:merge} ensure that the
affected facts are added to $\ipos \cup J$. Thus, $\mathsf{Close}(C_R)
\subseteq R[\ipos] \cup J$ holds after the update. For the $\supseteq$
direction, please note that the algorithm ensures that for each fact $R(u,v)
\in R[\Delta] \cup J$, $u$ and $v$ are in the same component in $C_R$ after the
update; moreover, $R[I \setminus \Delta] \subseteq \mathsf{Close}(C_R)$ already
holds before the update; so the $\supseteq$ direction of the property
$\mathsf{Close}(C_R) = R[\ipos] \cup J$ also holds after the update. Next we
show that $J = \semi{\Pi^{\mathsf{stc}(R)}}{\ipos}{\ineg}{\Delta}$ holds.
Before the function is executed, we have $\mathsf{Close}(C_R) = R[\ipos
\setminus \Delta]$, which implies $\mat{\Pi^{\mathsf{stc}(R)}}{\ipos \setminus
\Delta} = R[\ipos \setminus \Delta]$; similarly, after the function call we
have $\mathsf{Close}(C_R) = R[\ipos] \cup J$, which implies
$\mat{\Pi^{\mathsf{stc}(R)}}{\ipos} = \mat{\Pi^{\mathsf{stc}(R)}}{\ipos \cup J}
= R[\ipos] \cup J$. Moreover, it can be easily verified by induction on the
construction of $\semi{\Pi^{\mathsf{stc}(R)}}{\ipos}{\ineg}{\Delta}$ that
$\mat{\Pi^{\mathsf{stc}(R)}}{\ipos \setminus \Delta} \cup R[\Delta] \cup
\semi{\Pi^{\mathsf{stc}(R)}}{\ipos}{\ineg}{\Delta} =
\mat{\Pi^{\mathsf{stc}(R)}}{\ipos}$---that is, the seminaive computation
correctly closes $\ipos$ with respect to $\Pi^{\mathsf{stc}(R)}$. Therefore, we
have $J = \semi{\Pi^{\mathsf{stc}(R)}}{\ipos}{\ineg}{\Delta}$, as required.
\end{proof}

The fact that $\addfn{\mathsf{stc}(R)}$ captures $\Pi^{\mathsf{stc}(R)}$ during
the execution of Algorithm~\ref{and:matmod} directly follows from
Claim~\ref{claimadd} and property~\eqref{semicomputesadd}. Next we show that
during the execution of Algorithm~\ref{alg:dredmod}, functions
$\addfn{\mathsf{stc}(R)}$, $\delfn{\mathsf{stc}(R)}$,
$\difffn{\mathsf{stc}(R)}$, and $\redfn{\mathsf{stc}(R)}$ capture
$\Pi^{\mathsf{stc}(R)}$. Note that the order of the function calls is important
for correctness, so we examine $\difffn{\mathsf{stc}(R)}$ first. For each call
to function $\diff{\mathsf{stc}(R)}{\ipos}{\Deltapos}{\Deltaneg}$ in
line~\ref{alg:dredmod:reddel:diff}, we have $\Deltapos = D \setminus A$, which
contains only facts from previous strata. Thus
$\apply{\Pi^{\mathsf{stc}(R)}}{\ipos \appargs \Deltapos, \Deltaneg} =
\emptyset$ clearly holds since both rules in $\Pi^{\mathsf{stc}(R)}$ are
recursive. Our implementation for
$\diff{\mathsf{stc}(R)}{\ipos}{\Deltapos}{\Deltaneg}$ always return empty set
as well, so by Definition~\ref{def:diff} $\difffn{\mathsf{stc}(R)}$ captures
$\Pi^{\mathsf{stc}(R)}$ for the calls in line~\ref{alg:dredmod:reddel:diff}.
For the same reason, $\difffn{\mathsf{stc}(R)}$ captures
$\Pi^{\mathsf{stc}(R)}$ for the calls in line~\ref{alg:dredmod:redinsert:diff}
as well. We next focus on $\delfn{\mathsf{stc}(R)}$, $\redfn{\mathsf{stc}(R)}$,
and $\addfn{\mathsf{stc}(R)}$.

\begin{claim}\label{claimdelstc}
    If ${\mathsf{Close}(C_R) \cup Y_R = R[\ipos]}$ holds before each call to
    $\del{\mathsf{stc}(R)}{\ipos}{\ineg}{\Delta}{\Cnr}$ in a run of
    Algorithm~\ref{alg:dredmod}, then $C_R$ and $Y_R$ are updated so that
    ${\mathsf{Close}(C_R) \cup Y_R = R[\ipos \setminus \Delta] \setminus J}$
    holds, where ${J = \del{\mathsf{stc}(R)}{\ipos}{\ineg}{\Delta}{\Cnr}}$.
    Moreover, $J_l \subseteq J \subseteq J_u$ holds, where $J_l$ and $J_u$ are
    the lower and upper bounds from Definition~\ref{def:del}, respectively.
\end{claim}

\begin{proof}
First we show that $\mathsf{Close}(C_R) \cup Y_R = R[\ipos \setminus \Delta]
\setminus J$ holds after the update. For the $\subseteq$ direction, consider
arbitrary $R(u,v) \in \mathsf{Close}(C_R) \cup Y_R$ after the update. If
$R(u,v) \in \mathsf{Close}(C_R)$, since the algorithm only removes components
from $C_R$, we have $R(u,v) \in \ipos$. Moreover,
lines~\ref{alg:delstc:loopcomponents} and \ref{alg:delstc:deltam}, and the fact
that the components in $C_R$ are disjoint ensure that the component containing
$u$ and $v$ is not removed from $C_R$ during the update. But then, by
line~\ref{alg:delstc:loopcomponents} and line~\ref{alg:delstc:incj} we know
that $R(u,v) \not\in \Delta \cup J$. Therefore, $R(u,v) \in R[\ipos \setminus
\Delta] \setminus J$ holds. If $R(u,v) \in Y_R$ after the update, then we have
$\Cnr(R(u,v)) > 0$, so line~\ref{alg:delstc:check} ensures that $R(u,v) \not\in
J$ holds; moreover, $R(u,v) \not\in \Delta$ holds since otherwise there exists
another $\delfn{\mathsf{stc}(R)}$ that violates the lower bound constraint.
Therefore, $\mathsf{Close}(C_R) \cup Y_R \subseteq R[\ipos \setminus \Delta]
\setminus J$ holds after the update.

For the $\supseteq$ direction, consider arbitrary $R(u,v) \in R[\ipos \setminus
\Delta] \setminus J$. By $R[\ipos \setminus \Delta] \setminus J \subseteq
R[\ipos]$ and $\mathsf{Close}(C_R) \cup Y_R = R[\ipos]$ we have $R(u,v) \in
\mathsf{Close}(C_R) \cup Y_R$ before the udpate. No fact is removed from $Y_R$
during the update, so if $R(u,v) \in Y_R$ before the udpate, the same still
holds after the update. Now if $R(u,v) \in \mathsf{Close}(C_R)$ before the
update, then there are two cases. If the component in $C_R$ containing $u$ and
$v$ is not removed, then clearly we still have $R(u,v) \in \mathsf{Close}(C_R)$
after the udpate. If the component is indeeded removed, then we have
$\Cnr(R(u,v)) > 0$ since otherwise we would have $R(u,v) \in \Delta \cup J$,
which leads to a contradiction. But then, $R(u,v)$ is added to $Y_R$, so
$R(u,v) \in \mathsf{Close}(C_R) \cup Y_R$ holds as well after the function call.

Now we show that $\invsemic{\Pi^{\mathsf{stc}(R)}}{\ipos}{\ineg}{\Delta}{\Cnr}
= J_l \subseteq J \subseteq J_u =
\invsemi{\Pi^{\mathsf{stc}(R)}}{\ipos}{\ineg}{\Delta}$ holds. For the
right-hand inclusion $J \subseteq J_u$, consider arbitrary $R(u,v) \in J$. The
fact is added to $J$ in line~\ref{alg:delstc:incj} since there exists a fact
$R(u',v')$ such that $R(u',v') \in \Delta$ holds, and $u',v'$ are in the same
component as $u,v$ before the function call. But then, by $\mathsf{Close}(C_R)
\cup Y_R = R[\ipos]$ we clearly have $R(u,u') \in \ipos$ and $R(v',v) \in
\ipos$. It is straightforward to verify by induction on the construction of
$\invsemi{\Pi^{\mathsf{stc}(R)}}{\ipos}{\ineg}{\Delta}$ that applicable rule
instances $R(u,u') \wedge R(u',v') \rightarrow R(u,v')$ and $R(u,v') \wedge
R(v',v) \rightarrow R(u,v)$ will be considered, so $R(u,v) \in J_u$ holds. For
the left-hand inclusion $J_l \subseteq J$, consider arbitrary $R(u,v) \in
\invsemic{\Pi^{\mathsf{stc}(R)}}{\ipos}{\ineg}{\Delta}{\Cnr}$, by the
definition of $\mathsf{InvSemi}$ and the fact that $\mathsf{Close}(C_R) \cup
Y_R = R[\ipos]$ holds before the update, we know that there exists a component
$U \in C_R$ that contains $u,v,u',v'$ such that $R(u',v') \in \Delta$. But
then, line~\ref{alg:delstc:loopcomponents} ensures that $U$ is deleted and
$R(u,v)$ is checked in line~\ref{alg:delstc:check}. $R(u,v) \in J_l$ ensures
$\Cnr(R(u,v)) = 0$, since otherwise $J_l$ would not satisfy the lower bound
condition in Definition~\ref{def:del}. Therefore, line~\ref{alg:delstc:incj}
ensures that $R(u,v)$ is in $J$ (and it is not in $\Delta$ due to $R(u,v) \in
J_l$ and the definition of $\mathsf{InvSemi}$).
\end{proof}

That $\delfn{\mathsf{stc}(R)}$ captures $\Pi^{\mathsf{stc}(R)}$ during the
execution of Algorithm~\ref{alg:dredmod} immediately follows from
Claim~\ref{claimdelstc} and Definition~\ref{def:del}.

\begin{claim}\label{claimredstc}
    If ${\mathsf{Close}(C_R) \cup Y_R = R[\ipos \setminus \Delta]}$ holds
    before each call to $\red{\mathsf{stc}(R)}{\ipos}{\ineg}{\Delta}$ in a run
    of Algorithm~\ref{alg:dredmod}, then $C_R$ is updated so that
    ${\mathsf{Close}(C_R) = R[\ipos \setminus \Delta] \cup J}$ holds, where ${J
    = \red{\mathsf{stc}(R)}{\ipos}{\ineg}{\Delta}}$. Moreover, $J$ is the
    smallest dataset satisfying ${\apply{\Pi^{\mathsf{stc}(R)}}{(\ipos
    \setminus \Delta) \cup J, \ineg} \cap \Delta \subseteq J}$.
\end{claim}

\begin{proof}
$\mathsf{Close}(C_R) \cup Y_R = R[\ipos \setminus \Delta]$ ensures
$\mathsf{Close}(C_R) = R[(\ipos \setminus \Delta) \setminus Y_R]$. But then,
$\mathsf{Close}(C_R) = R[\ipos \setminus \Delta] \cup J$ holds after the update
in the same way as in the proof for Claim~\ref{claimadd}. Moreover, $J$ is the
smallest dataset that satisfies
\begin{displaymath}
    \apply{\Pi^{\mathsf{stc}(R)}}{(\ipos \setminus \Delta) \cup J, \ineg \appargs Y_R \cup J} \subseteq (\ipos \setminus \Delta) \cup J.
\end{displaymath}
Now to see that $J$ is also the smallest dataset satisfying
$\apply{\Pi^{\mathsf{stc}(R)}}{(\ipos \setminus \Delta) \cup J, \ineg} \cap
\Delta \subseteq J$, assume for the sake of a contradiction that a smaller
dataset $J'$ satisfies the above---that is,
$\apply{\Pi^{\mathsf{stc}(R)}}{(\ipos \setminus \Delta) \cup J', \ineg} \cap
\Delta \subseteq J'$ holds. Now consider arbitrary $F \in
\apply{\Pi^{\mathsf{stc}(R)}}{(\ipos \setminus \Delta) \cup J', \ineg \appargs
Y_R \cup J'} \subseteq \apply{\Pi^{\mathsf{stc}(R)}}{(\ipos \setminus \Delta)
\cup J', \ineg}$. Since the original materialisation $I$ is passed as $\ipos$
when the function gets called, it is clear that $F \in \ipos = I$ holds. Now if
$F \not \in \Delta$, then by our assumption we have $F \in J'$. But then, $J'$
also satisfies ${\apply{\Pi^{\mathsf{stc}(R)}}{(\ipos \setminus \Delta) \cup
J', \ineg \appargs Y_R \cup J'} \subseteq (\ipos \setminus \Delta) \cup J'}$
and $J'$ is smaller than $J$, which is a contradiction.
\end{proof}

That $\redfn{\mathsf{stc}(R)}$ captures $\Pi^{\mathsf{stc}(R)}$ during the
execution of Algorithm~\ref{alg:dredmod} follows from Claim~\ref{claimredstc}
and Definition~\ref{def:red}. Finally, the proof of $\addfn{\mathsf{stc}(R)}$
capturing $\Pi^{\mathsf{stc}(R)}$ during the execution of
Algorithm~\ref{alg:dredmod} is analogous to the proof of Claim~\ref{claimadd}
so we omit the details for the sake of brevity.

}{}

\end{document}